%% file: arxiv-main.tex
\pdfoutput = 1
\documentclass[11pt]{article}
\input{macros}

\newcommand{\citet}{\cite}
\newcommand{\citep}{\cite}

\title{On Computationally Efficient Multi-Class Calibration}
\date{}
\author{
Parikshit Gopalan \\
Apple\\
{\tt parikg@apple.com}
\and
Lunjia Hu\thanks{Part of this work done while LH was interning at Apple. LH is also supported by Moses Charikar’s and Omer Reingold’s Simons Investigators awards, Omer Reingold’s NSF Award IIS-1908774, and the Simons Foundation Collaboration on the Theory of Algorithmic Fairness.}\\
Stanford University\\
{\tt lunjia@stanford.edu}
\and
Guy N. Rothblum \\
Apple\\
{\tt rothblum@alum.mit.edu}
}

\allowdisplaybreaks
\begin{document}

\maketitle
\begin{abstract}
\input{./abstract}

\end{abstract}
\thispagestyle{empty}
\newpage
\tableofcontents
\thispagestyle{empty}
\newpage
\setcounter{page}{1}

\input{./intro}
\input{./new-defs}

\input{./new-equiv}

\input{./hardness}

\input{./upper}

\input{./smooth}

\input{./sigmoid}
\bibliographystyle{alpha}
\bibliography{arxiv-refs.bib}

\appendix

\input{./appendix}

\end{document}

%% file: macros.tex
\usepackage[utf8]{inputenc}
\usepackage{amsmath, amssymb, amsthm, bbm, graphicx, url}
\usepackage{amsfonts,fullpage}
\usepackage{float}
\usepackage{graphicx}

\usepackage{bbm}
\usepackage{xcolor}
\usepackage[boxruled,linesnumbered,vlined]{algorithm2e}
\usepackage[pagebackref]{hyperref}
\hypersetup{
    colorlinks=true,
    linktocpage=true,
    linkcolor=blue!50!black,  %
    citecolor=blue!50!black,  %
    urlcolor=blue!50!black    %
}
\usepackage{cleveref}
\usepackage{dsfont}
\usepackage{xspace}
\usepackage{comment}
\usepackage{tikz}

\newcommand{\eat}[1]{}

\newtheorem*{definition*}{Definition}
\newtheorem*{proposition*}{Proposition}
\newtheorem*{corollary*}{Corollary}

\newtheorem{theorem}{Theorem}[section]
\newtheorem{lemma}[theorem]{Lemma}
\newtheorem{definition}[theorem]{Definition}

\newtheorem{corollary}[theorem]{Corollary}

\newtheorem{assumption}[theorem]{Assumption}

\newtheorem{claim}[theorem]{Claim}

\newcommand{\wh}{\widehat}
\newcommand{\Z}{\mathbb{Z}}
\newcommand{\R}{\mathbb{R}}
\newcommand{\X}{\mathcal{X}}

\newcommand{\mD}{\mathcal{D}}
\newcommand{\mH}{\mathcal{H}}

\newcommand{\mT}{\mathcal{T}}
\newcommand{\mB}{\mathcal{B}}

\newcommand{\mW}{\mathcal{W}}

\newcommand{\mU}{\mathcal{U}}
\newcommand{\mE}{\mathcal{E}}

\newcommand{\tp}{p}

\newcommand{\sy}{\mathbf{y^*}}
\newcommand{\Ek}{\mathcal{E}_k}
\newcommand{\Dk}{\Delta_k}

\newcommand{\x}{\mathbf{x}}
\newcommand{\z}{\mathbf{z}}

\newcommand{\y}{\mathbf{y}}
\newcommand{\bv}{\mathbf{v}}
\newcommand{\bdv}{\mathbf{v}}

\newcommand{\lt}{\left}
\newcommand{\rt}{\right}

\newcommand{\zo}{\ensuremath{\{0,1\}}}

\newcommand{\norm}[1]{\left\lVert#1\right\rVert}

\renewcommand{\hat}{\wh}
\newcommand{\eps}{\varepsilon}

\newcommand{\abs}[1]{\ensuremath \Bigl\lvert #1 \Bigr\rvert}

\newcommand{\ip}[2]{\ensuremath{\langle #1, #2 \rangle}}

\newcommand{\GR}[1]{}%
\newcommand{\PG}[1]{}%
\newcommand{\LH}[1]{}%

\newcommand{\Xor}{\mathrm{Xor}}
\newcommand{\XS}{\mathrm{XorSat}}
\newcommand{\mA}{\mathcal{A}}
\newcommand{\mR}{\mathcal{R}}

\newcommand{\pmo}{\ensuremath{ \{\pm 1\} }}
\newcommand{\fr}[1]{\ensuremath{\frac{1}{#1}}}

\newcommand{\opt}{\mathrm{Opt}}
\newcommand{\val}{\mathrm{val}}
\newcommand{\ind}[1]{\ensuremath{\mathbf{1}_{#1}}}
\newcommand{\I}[1]{\ensuremath{\mathbb{I}(#1)}}

\newcommand{\defeq}{{:=}}

\DeclareMathOperator{\poly}{poly}

\DeclareMathOperator*{\E}{\mathbf{E}}

\newcommand{\Lip}{\mathrm{Lip}}
\newcommand{\pLip}{\mathrm{pLip}}
\newcommand{\fLip}{\mathrm{fLip}}
\newcommand{\hplip}[1]{{\mathcal H}_{#1\text-\pLip}}
\newcommand{\hflip}{{\mathcal H}_{\fLip}}
\newcommand{\hs}{{\mathsf{hs}}}

\newcommand{\smCE}{\mathrm{smCE}}
\newcommand{\ssCE}{\mathrm{ssCE}}
\newcommand{\psCE}{\mathrm{psCE}}
\newcommand{\fsCE}{\mathrm{fsCE}}

\newcommand{\mpk}[1]{}

\newcommand{\sign}{\mathrm{sign}}

\newcommand{\bE}{{\mathbb{E}}}

\newcommand{\ECE}{\mathsf{ECE}}

\newcommand{\CE}{\mathsf{CE}}

\newcommand{\bR}{{\mathbb R}}
\newcommand{\e}{{\mathbf e}}

\newcommand{\lift}{{\mathsf{lift}}}

\newcommand{\psc}{{\mathsf{PSC}}}
\newcommand{\bg}{{\mathbf g}}
\newcommand{\br}{{\mathbf r}}
\newcommand{\bs}{{\mathbf s}}

\newcommand{\one}{{\mathbb I}}
\newcommand{\kn}{{\mathsf{ker}}}
\newcommand{\acc}{{``accept''}}
\newcommand{\rej}{{``reject''}}

\newcommand{\sps}[1]{^{(#1)}}
\newcommand{\sset}[1]{{#1\text{-}\mathsf{ss}}}
\newcommand{\ba}{{\mathbf a}}
\newcommand{\bu}{{\mathbf u}}
\newcommand{\ext}{{\mathsf{ext}}}
\newcommand{\dce}{{\mathsf{dCE}}}
\newcommand{\dec}{{\mathsf{decCE}}}

\newcommand{\lunjia}[1]{{\color{red}[LH: #1]}}
\newcommand{\guy}[1]{}
\newcommand{\parik}[1]{}

\newcommand{\CT}{{\mathcal T}}

%% file: abstract.tex
Consider a multi-class labelling problem, where the labels can take values in $[k]$, and a predictor predicts a distribution over the labels. In this work, we study the following foundational question: \emph{Are there notions of multi-class calibration that give strong guarantees of meaningful predictions and can be achieved in time and sample complexities polynomial in $k$?} Prior notions of calibration exhibit a tradeoff between computational efficiency and expressivity: they either suffer from having sample complexity exponential in $k$, or needing to solve computationally intractable problems, or give rather weak guarantees. 

Our main contribution is a notion of calibration that achieves all these desiderata: we formulate a robust notion of \emph{projected smooth calibration} for multi-class predictions, and give new recalibration algorithms for efficiently calibrating predictors under this definition with complexity polynomial in $k$. Projected smooth calibration gives strong guarantees for all downstream decision makers who want to use the predictor for binary classification problems of the form: does the label belong to a subset $T \subseteq [k]$: \emph{e.g. is this an image of an animal?} It ensures that the probabilities predicted by summing the probabilities assigned to labels in $T$ are close to some perfectly calibrated binary predictor for that task. We also show that natural strengthenings of our definition are computationally hard to achieve: they run into information theoretic barriers or computational intractability. 

Underlying both our upper and lower bounds is a tight connection that we prove between multi-class calibration and the well-studied problem of agnostic learning in the (standard) binary prediction setting. This allows us to use kernel methods to design efficient algorithms, and also to use known hardness results for agnostic learning based on the hardness of refuting random CSPs to show lower bounds.

%% file: intro.tex
\section{Introduction}
\label{sec:intro}

The ubiquitous use of machine learning for making consequential decisions has resulted in a renewed interest in the question {\em what should probabilistic predictions mean}? This question has a long history going back at least as far as the literature on forecasting \citep{Dawid, dawid1984present}. Calibration is a classical interpretability notion for binary predictions originating in this setting that is widely used in modern machine learning. In the binary classification setting, denoting the label $\y \in \zo$ and the predicted probability of $1$ by $\bv \in [0,1]$, (perfect) calibration requires $\E[\y|\bv] = \bv$. 

There has been renewed research interest both in the calibration of modern DNNs \citep{guo2017calibration} and in foundational questions about how best to define and measure calibration to ensure robustness and efficiency \citep{utc1, KLST23} building on earlier work of \cite{kakadeF08}.
We study calibration notions in the context of multi-class classification, where the goal is to assign one of $k$ possible labels to each input. A predictor assigns to each input a distribution over the labels, which allows it to convey uncertainty in its predictions. Values of $k$ in the thousands are increasingly common, especially for vision tasks \cite{imagenet}, so the efficiency in terms of the parameter $k$ is increasingly relevant. In this setting, even the right definition of calibration is not immediate. There are a multitude of existing definitions in theory and practice, such as confidence \cite{guo2017calibration}, class-wise \cite{kull2019beyond}, distribution \cite{KullF15} and decision \cite{zhao2021calibrating} calibration. However, existing notions %
either provide only weak guarantees for meaningful predictions, are computationally hard to achieve, or are even information theoretically hard to achieve, requiring exponential sample complexity  in $k$.

In this work, we study the following foundational question:
\begin{center}
{\em Are there notions of multi-class calibration that give strong guarantees of meaningful predictions and can also be achieved with time and sample complexities polynomial in $k$?}
\end{center}

Our main contribution is answering this question in the affirmative: we formulate a robust notion of {\em projected smooth calibration} for multi-class predictions, and give new recalibration algorithms\footnote{The exact notion of calibrating a predictor has to be defined carefully to avoid trivial solutions (for example, the constant predictor that always outputs the empirical mean is perfectly calibrated). Following much of the literature, our algorithms post-process a given predictor to make it calibrated while not increasing the squared loss.}  for efficiently calibrating predictors under this definition (and variants of it). We also show that natural strengthenings of this definition are computationally  or information-theoretically hard to achieve. An important ingredient in showing these new upper and lower bounds is a tight connection between multi-class calibration and the well-studied problem of agnostic learning in the (standard) binary prediction setting. We proceed to elaborate on the setting, prior work, and our contributions.

\paragraph{Multi-class calibration.} In the $k$-class prediction setting, we have an underlying distribution over instance-outcome pairs, where we view the outcome $\y$ as the one-hot encoding of a label from $[k]$. A prediction vector $\bv \in \Delta_k$ describes a distribution in the $k$-dimensional simplex, where a perfect prediction would describe the exact distribution of the outcome $\y$ for that instance. {\em Canonical calibration}  \cite{KullF15}, also called distribution calibration, is the most stringent notion, which requires that $\E[\y|\bv] = \bv$ (the expectation averages over all instances for which the prediction is $\bv$). The naive procedure for checking whether canonical calibration holds even approximately requires (after suitable discretization) conditioning on $\exp(k)$ many possible predictions in $\Delta_k$. Indeed, %
we show that even the easier problem of distinguishing a perfectly calibrated predictor from one that is far from calibrated requires $\exp(k)$ samples.
At the other extreme, {\em class-wise calibration } \cite{kull2019beyond} %
only requires that for every $i \in [k]$, $\E[\y_i|\bv_i] = \bv_i$. This notion can be achieved efficiently, but we argue below that it is not sufficiently expressive. 

Assume that we have a class-wise calibrated predictor and we wish to use it for downstream binary classification tasks. For instance, we might want to classify images as being those of animals, where {\em animals} is a subset of labels. Assume for simplicity that $c$ for {\em cat} and $d$ for {\em dog} are the only animals in our $k$ labels. \guy{tried to simplify:} %
Class-wise calibration   ensures that the predicted probabilities $\bv_c, \bv_d \in [0,1]$ are each calibrated on their own: conditioned on, say, the predicted probability of cat being 0.2, the outcome should be a cat w.p.\ roughly 0.2. Suppose, however, that we want to predict whether the image is a cat {\em or} a dog. The natural probability to predict is $\bv_c + \bv_d$, but this might be far from calibrated w.r.t the actual probability that the outcome is a cat or a dog (even though the predictor is class-wise calibrated).
This reveals a weakness of class-wise calibration that is also shared by other guarantees that we know how to achieve efficiently (such as {\em confidence calibration} \cite{guo2017calibration}, see below): their calibration guarantees are rather fragile, and break down when used in downstream tasks.

Aiming to achieve rigorous downstream guarantees, Zhao {\em et al.} \cite{zhao2021calibrating} introduced {\em decision calibration}, which
can be achieved in $\poly(k)$ sample complexity.\footnote{The paper claims the notion is both time and space efficient, but their main result \cite[Theorem 2]{zhao2021calibrating} only proves a bound on sample complexity. See the discussion in Sections \ref{subsec:related-work} and \ref{sec:defs} and \ref{sec:hardness}.} However we show that the algorithmic task they aim to solve is as hard as agnostically learning halfspaces, and hence is unlikely to be achievable in time $\poly(k)$ by results of \cite{amit-daniely}. 

To summarize, the state of the art for multiclass calibration notions:
\begin{itemize}
    \item There are {\bf efficient} notions, such as  classwise and confidence calibration, but they are not very expressive. In particular their calibration guarantees are rather fragile and do not imply good guarantees for downstream tasks.
    
    \item There are {\bf expressive} notions, such as canonical calibration and the recently proposed notion of decision calibration, but they are inefficient. These notions run into information or complexity theoretic barriers, which prevent them from being achievable in running time and sample complexity $\poly(k)$. 
\end{itemize}

This motivates our foundational question: is there an {\em expressive and efficient} notion of multi-class calibration? Such a notion should give robust calibration guarantees for downstream tasks, and should be achievable in $\poly(k)$ time and sample complexity. More broadly, is there a general framework for understanding the complexity of various calibration notions? Ideally, such a framework would let us identify broad classes of notions that are  efficiently achievable and identify computational and information-theoretic barriers to other notions.

These questions are motivated not only by the use of calibration as an notion of interpretability for probabilistic predictions in machine learning, but also by the recent applications of calibration to fairness \cite{hkrr2018}, loss minimization \cite{omni, lossOI} and indistinguishability \cite{OI, dwork2022beyond,back}. In the multi-class setting with $k$ labels, algorithms for all of these notions become exponential in $k$, which stems from the fact that they try to achieve canonical calibration or similarly expressive notions (see for example \cite{omni, dwork2022beyond}). We see formulating more efficient notions of calibration as a step towards more efficient algorithms for these applications in the multiclass setting.

\subsection{Our Contributions}

We start by describing a unifying framework from \cite{GopalanKSZ22} for various notions of multiclass calibration, for which we need some notation. 
Let $\Delta_k\subseteq\R^k$ denote the probability simplex for $k$ outcomes. 
Given a distribution $\mD_0$ on $(\x, \y)$ pairs where $\y \in \zo^k$ is the one-hot encoding of a label and a predictor $p$, let $\bv = p(\x) \in \Delta_k$ be the prediction of $p$. 
Let $\mD$ denote the induced distribution of $(\bv,\y)$.

\paragraph{Weighted calibration.}

As observed by various works \cite{OI, GopalanKSZ22, dwork2022beyond, utc1}, calibration is essentially a notion of indistinguishability of distributions. For multiclass learning, perfect canonical calibration requires that for every $\bv \in \Delta_k$,  $\E[\y|\bv]$ (which completely describes the distribution of $\y$ conditioned on $\bv$) equals $\bv$. If we relax equality to expected closeness in $\ell_1$ distance $\E_\bv[\left| \E[\y|\bv] - \bv \right|]$, we arrive at the notion of the expected calibration error or $\ECE$. This notion requires $\exp(k)$ samples to estimate (see Theorem \ref{thm:canonical}); %
it is also not robust to small perturbations of the predictor $p$ \cite{kakadeF08, utc1}. 
We aim for relaxed calibration notions that capture the same underlying principle that $\E[\y|\bv]$ and $\bv$ are ``close'' under $\mD$, but which are efficient to estimate, and also do not suffer from the same kind of non-robustness. 

Following \cite{GopalanKSZ22}, we work with the definition of weighted calibration, which is general enough to capture all the aforementioned notions of calibration.
  For a hypothesis class $\mH: =\{h: \Delta_k \to [-1,1]\}$, we consider the family of {\em weight functions} $\mH^k$ mapping $\Delta_k \to [-1,1]^k$, where for every $i\in [k]$, coordinate $i$ of the output is a function $h_i\in \mH$ of the input.  Define the weighted calibration error as
\[ \CE_{\mH^k}(\mD) \defeq \max_{w \in \mH^k}\left|\E_{(\bv,\y)\sim \mD}[\ip{w(\bv)}{\y - \bv}]\right| = \max_{w \in \mH^k}\left|\E_{\bv }[\ip{w(\bv)}{\E[\y|\bv] - \bv}]\right|.\]

This can be seen as requiring closeness of the distributions of $\bv$ and $\E[\y|\bv]$ to the class of distinguishers $\mH^k$ in the spirit of pseudorandomness. Taking $\mH$ to be all functions on $\Delta_k$ bounded by $1$ in absolute value recovers the notion of $\ECE$. %
Relaxing the space of distinguishers weakens the definition. Are there distinguisher families where the calibration guarantee remains meaningful, while simultaneously allowing for efficient {\em auditing}: deciding whether a given predictor satisfies $\CE_{\mH^k}(\mD) \leq \alpha$?

\eat{Underlying our algorithms and hardness results is a tight 
characterization of efficient auditing in terms of agnostic learning. 
Weak agnostic learning for a class $\mH$ is a standard learning problem in the binary (not multiclass) classification setting, where we have a distibution on $\Dk \times \pmo$ and the goal is to decide is there exists $\E_{(\bv, z) \sim \mD} [h(\bv)z] \geq \gamma$. We show that {\em auditing for $\mH^k$ is efficient iff the class $\mH$ is is weakly agnostically learnable.} 
\lunjia{We need to somehow emphasize that this characterization is not straightforward, i.e.,\ difference between the multi-class setting and the binary setting, not losing a factor of $k$ in alpha...}}

\paragraph{Projected smooth calibration.} 
We now formulate projected smooth calibration, a weighted calibration notion that satisfies our desiderata.
 As discussed above, we want to ensure the following {\bf subset calibration} guarantee:
for every subset $T \subseteq [k]$ of labels, the probabilities assigned by our predictor to the event that the label belongs to $T$ should be calibrated. Let $\bv \in \Dk$ denote the prediction of our predictor. Letting  $\ind T \subseteq \zo^k$ denote the indicator vector of $T$, the %
indicator for the event that the outcome $\y$ is in $T$ is $\ind T \cdot \y$, whereas the predicted probability is $\ind T \cdot \bv$. 
Say we want to enforce the  calibration condition that when the predicted probability of belonging to $T$ exceeds $v \in [0,1]$, the label indeed lies in $T$ with roughly the predicted probability.  We can view this as requiring a bound on
\begin{align*} 
\abs{ \E[\I{\ind T \cdot \bv \geq v}(\ind T\cdot \y - \ind T \cdot \bv)] } = \abs{ \E[ \I{\ind T \cdot \bv \geq v}\ind T\cdot (\y - \bv)] }
\end{align*} 
where $( \I{\ind T\cdot \bv \geq v}\ind T) \in \zo^k$ is a vector-valued function on $\Dk$, which takes the value $\I{\ind T\cdot \bv \geq v}$ for coordinates in $T$, and the value $0$ for coordinates outside $T$. The good news is that this setup fits the template of weighted calibration, where the class $\mH$ contains all functions the form $\I{\ind T\cdot \bv \geq v}$ for $T \in \zo^k, v \in \R$. The bad news is that we will show this problem is as hard as agnostically learning halfspaces. Daniely \cite{amit-daniely} showed that, assuming the hardness of refuting random $\Xor$ CSPs, this problem cannot be solved in polynomial time. 

In projected smooth calibration, we replace the hard thresholds $\I{\ind T\cdot \bv \geq v}$ with the class $\mH_\pLip =\{ \phi(\ba\cdot \bv)\}$ where $\phi: [-1,1] \to [-1,1]$ is a Lipschitz continuous function and $\ba \in [-1,1]^k$.
In particular, this includes indicator vectors for subsets. 
Projected smooth calibration requires that the weighted calibration error for the weight family $\mH_\pLip^k$ is bounded.

\begin{definition}[Projected smooth calibration, informal statement of \Cref{def:p-smooth}]
For a joint distribution $\mD$ on predictions $\bv$ and true outcomes $\y$, the \emph{projected smooth calibration error} is
\[
\psCE(\mD):=\sup_{w\in \mH_\pLip^k} \left|\E_{(\bv,\y)\sim \mD}[(\y - \bv)w(\bv)]\right|,
\]
where $\mH_\pLip^k$ is the class of functions $w:\Dk\to [-1,1]^k$ such that for every coordinate $i$ in $w$'s output, denoted by $w\sps i$, there are a $1$-Lipschitz function $\phi\sps i:[-1,1]\to [-1,1]$ 
and a vector $\ba_i\in [-1,1]^k$
s.t.
$w\sps i(\bv) = \phi\sps i (\ba_i\cdot \bv)$ for every $\bv\in \Delta_k$.
\end{definition}

A predictor satisfies projected smooth calibration if its error $\psCE$ is bounded. We show that this definition satisfies all our desiderata:

\medskip \noindent {\bf Property (1): expressive power.} Projected smooth calibration guarantees that for every subset $T \subseteq [k]$, the predicted probabilities for the binary classification task (namely, the outcome is in $T$) satisfy smooth calibration, a well-studied calibration notion with several desirable properties \cite{kakadeF08, GopalanKSZ22, utc1}. In particular, this implies that the predicted probabilities $\ind T \cdot \bv$
that the outcome will be in $T$ are close to being perfectly calibrated.  The proof builds on the work of \cite{utc1}. In particular, for each such subset $T$, there exists a perfectly calibrated predictor $p^*_T$ for the binary classification task of determining whether the outcome will land in $T$, whose predictions are close to $\ind T\cdot \bv$ in earthmover distance. Thus, we get meaningful guarantees for a rich collection of downstream binary classification tasks (including subset membership and more).

\eat{Crucially, projected smooth calibration does {\em not} perform a ``sharp'' thresholding on the probabilities (as discussed above, sharp thresholds lead to intractability). Thus, it might be that if we condition on the instances where the predicted probability of landing in $T$ is {\em exactly} above $v$, then the probability that the outcome lands in $T$ is far from the predictions. However, by ``massaging'' the predictions a bit (and moving them away from the sharp threshold $v$), we can get to perfectly calibrated predictions---a strong guarantee.}

\medskip \noindent {\bf Property (2): Computational efficiency.} We show an efficient algorithm for auditing whether the projected smooth calibration error of a predictor is bounded.  Here we state our result informally, the formal statement is in Theorem \ref{thm:p-smooth}.

\begin{theorem}[Efficient auditing, informal statement of \Cref{thm:p-smooth}]
\label{thm:projected-smooth-intro}
There is an algorithm for deciding whether the projected smooth calibration error is at most $\alpha$, with sample complexity and running time $O(k^{O(1/\alpha)})$. 
\end{theorem}
The work of \cite{SSS11} showed that  agnostic learning halfspaces becomes tractable if we replace the hard thresholds used in halfspaces with Lipschitz transfer functions. Building on their techniques, and using Jackson's Theorem on low-degree uniform approximations for Lipschitz functions, we show that auditing for projected smooth calibration is polynomial time solvable. Moreover, our auditing algorithm is quite simple and does not need to solve a convex program,  generalizing results in \cite{KSJ18, utc1}. Our algorithm in fact solves the associated search problem (see Definition \ref{def:audit-guarantee}): if $p$ is not calibrated, it finds a witness to the lack of calibration, which can be used to post-process $p$ and reduce its calibration error without increasing the squared loss (see \Cref{lem:audit}). Defining a recalibration algorithm correctly is subtle, see Definition \ref{def:audit-guarantee}, \Cref{lem:audit} and the discussion around them. 

Our algorithm has running time polynomial in $k$ for every fixed constant $\alpha$, in contrast with previous results for expressive notions of calibration.  If we only care about auditing using $\phi(w\cdot \bv)$ for vectors $w$ where $\norm{w}^2_2 \leq m$ then the sample complexity can be bounded by $m^{O(1/\alpha)}$. This gives a running time  fixed polynomial in $k$ (but exponential in $\alpha$) for subset calibration where we only care about bounded size subsets. 
One can also get better run times by restricting the family of Lipschitz functions $\phi$. By restricting to auditors of the form $\tanh(w\cdot \bv)$ we get the weaker notion of {\bf sigmoid calibration}, for which the auditor runs in time $k^{O(\log(1/\alpha))}$. This can be seen as a smooth relaxation of the intractable notion of halfspace auditing. However, the improved efficiency comes at the price of some expressivity, we do not get closeness to perfect calibration for downstream subset classification tasks. 

It is interesting to investigate whether the exponential dependence on $1/\alpha$ in \Cref{thm:projected-smooth-intro} can be avoided. As a result in this direction, we show that the running time cannot be improved to $\poly(k,1/\alpha)$:
\begin{theorem}[Informal statement of \Cref{thm:hard-psmooth}]
Under standard complexity-theoretic assumptions,
there is no algorithm that can decide whether the projected smooth calibration error is at most $\alpha$ with sample complexity and running time $k^{O(\log^{0.99}(1/\alpha))}$. 
\end{theorem}
 We prove the theorem by showing a reduction from the task of refuting random XOR formulas. Getting the right exponent for $k$ as a function of $1/\alpha$ is an interesting question for future work.

\medskip \noindent{\bf Property (3): Robustness.} The works of \cite{kakadeF08, FosterH18, utc1} advocate the use of Lipschitz functions in defining calibration since it results in robust measures that do not change drastically under small perturbations of the predictor. Since projected smooth calibration is defined using Lipschitz functions, it is a robust calibration measure.

\paragraph{Lower bounds for stronger notions.} The discussion above suggests possible strengthenings of the notion of projected smooth calibration. The weight function family $\phi(w\cdot \bv)$ is a subset of the family $\fLip$ of all Lipschitz functions $\psi: \Dk \to [-1,1]$. We could imagine using $\fLip^k$ as our weight function family to get a stronger notion of calibration, which we call \emph{full smooth calibration}. We show in Theorem \ref{thm:full} that this notion is information-theoretically intractable and requires $\exp(k)$ samples.

\begin{theorem}[Informal statement of \Cref{thm:full}]
Any algorithm to decide whether the full smooth calibration error is 0 or exceeds a positive
absolute constant requires $\exp(k)$ samples.
\end{theorem}

In our closeness to calibration guarantee, for every $T \subset [k]$ there exists a binary predictor $p^*_T$  whose predictions are close to $\ind T\cdot \bv$ and which is perfectly calibrated. But the different $p^*_T$ for various $T$s might not be consistent, meaning they need not arise as $p^*_T =\ind T\cdot p^*$ where $p^*$ is a perfectly calibrated predictor (independent of the choice of $T$). Could we instead measure calibration error by comparing our predictions to those made by a single calibrated predictor $p^*$? Put differently, projected smooth calibration guarantees that our predictions on each subset $T$ are {\em locally} close to the predictions of a calibrated binary predictor $p^*_T$. We are now asking whether one can measure {\em global} closeness to a single perfectly calibrated predictor $p^*$.   

There is prior work that suggests measuring closeness to calibration in terms of distance of its predictions from the nearest perfectly calibrated predictor $\bv^*$. This notion, called distance to calibration, was studied in the binary setting by \cite{utc1}, where it plays a central role in their theory of {\em consistent calibration measures}. Such measures are ones that approximate the distance to calibration within polynomial factors. They identified several efficiently computable consistent calibration measures in the binary setting, including smooth calibration and Laplace kernel calibration. 

We show a strong negative result for measuring or even weakly approximating the distance to calibration in the multiclass prediction setting: 

\begin{theorem}[Informal statement of \Cref{thm:canonical}]
\label{thm:intro-distance}
Any algorithm to decide whether the distance to calibration is $0$ or exceeds a positive absolute constant requires $\exp(k)$ samples. 
\end{theorem}

In contrast to the work of \cite{utc1}, \Cref{thm:intro-distance} shows that in the multiclass setting, any consistent calibration measure requires $\exp(k)$ samples.

These lower bounds stem from an indistinguishability argument that we sketch below. We take $V \subset \Dk$ of size $\exp(k)$ of predictions that are $\Omega(1)$ far from each other. We construct two distributions $\mD_1$ and $\mD_2$ on predictions and labels. In either distribution, the marginal distribution on predictions $\bv$ is uniform on $V$. In $\mD_1$, the distribution on labels $\y_1$ conditioned on $\bv$ is perfectly calibrated, so that $\E[\y_1|\bv] = \bv$. In $\mD_2$, for each $\bv \in V$, the label $\y_2$ is fixed to be some single value. We imagine this label $\y_2$ being picked at random with $\E[\y_2|\bv] = \bv$ the very first time we predict $\bv$. Every subsequent time we see the prediction $\bv$, we will see this same label $\y_2|\bv$.  A sampling algorithm cannot tell the difference between these distributions until it sees multiple samples with the same value of $\bv \in V$, since until that point, samples from the two distributions are identically distributed. By the birthday paradox, this requires  $\Omega(\sqrt{|V|}) = \exp(k)$ samples.

\paragraph{Equivalence between auditing and agnostic learning.} Underlying our algorithms and hardness results is a tight 
characterization of efficient auditing in terms of agnostic learning.
We elaborate on these two computational tasks. The auditing task for the class of weight functions $\mH^k$ gets as input a predictor $p$, and needs to decide whether it has large calibration error for $\mH^k$. If so, then the auditor should also return a weight function $w'$ that has large calibration error (in the spirit of weak agnostic learning \cite{SBD2, KalaiMV08}, we allow for a gap between the largest calibration error in $\mH^k$ and the error of the weight function found by the auditor). As noted in the discussion following Theorem \ref{thm:projected-smooth-intro}, solving the auditing task also allows us to efficiently recalibrate a given predictor to achieve low weighted calibration error for the class $\mH^k$.
Weak agnostic learning for a class $\mH$ is a standard learning problem in the binary (not multi-class) classification setting, where we have a distribution on $\Dk \times \pmo$ and the goal is to find a witness given the existence of $h\in \mH$ with correlation $\E_{(\bv, z) \sim \mD} [h(\bv)z]$ at least $\gamma$. We show that {\em auditing for $\mH^k$ is efficient iff the class $\mH$ is efficiently weakly agnostically learnable.}

\begin{theorem}[Informal statement of \Cref{thm:red,thm:product-hardness}]
Auditing for $\mH^k$ and agnostic learning for $\mH$ reduce to each other efficiently.\footnote{
We use an auditor for a slightly different class $\tilde\mH^k$ to solve the learning task for $\mH$, where $\tilde \mH$ is obtained from $\mH$ by taking a simple affine transformation of the input. In particular, the two classes are the same when $\mH$ is the class of halfspaces.
}
The calibration error parameter in auditing corresponds to the correlation parameter in learning up to a constant factor.
\end{theorem}

\guy{rephrased:} Connections between auditing for calibration and agnostic learning have appeared in \cite{hkrr2018} \guy{other citations?} and subsequent works. The focus was on binary or scalar prediction tasks, where the challenge is guaranteeing calibration for many different subsets of the feature vectors. The challenge in our work is different: we aim to guarantee calibration w.r.t. the $k$-dimensional multi-class outcome vector $\y$, and to relate this task to agnostic learning with binary labels.
As we show in \Cref{sec:red-algo}, applying a learning algorithm to an auditing task in a coordinate-wise manner would result in losing a factor of $k$ in the calibration error. This loss of $k$ would result in auditing algorithms that do not run in time $\poly(k)$ even for constant $\alpha$. We show that this loss can always be avoided by applying the learning algorithm on carefully constructed conditional distributions, giving a tight connection up to constant factors in the calibration error.

The equivalence between auditing and learning allows us to apply a rich set of techniques from the literature for agnostic learning to show both hardness results and efficient algorithms for auditing tasks. In particular, our hardness result for auditing decision calibration (\Cref{thm:hardness-dec}) is based on the hardness of agnostically learning halfspaces shown in previous work \cite{amit-daniely}. In general, our auditing algorithms can be instantiated with weight functions that have bounded norm in any reproducing kernel Hilbert space over $\Delta_k$, as long as the corresponding kernel can be evaluated efficiently. We apply polynomial approximation theorems and the multinomial kernel used in learning algorithms \cite{SSS11,GoelKKT17,GoelKK20} to give efficient auditors for projected smooth calibration and sigmoid calibration.

\subsection{Further Discussion of Related Work}
\label{subsec:related-work}

As discussed above, many works have discussed notions of calibration for multi-class prediction. These either offer limited expressiveness, or require super-polynomial runtime or sample complexities. We further elaborate on two recent works \cite{zhao2021calibrating, dwork2022beyond} that achieve polynomial sample complexity, but suffer from computational intractability. We also discuss the  work of \cite{GopalanKSZ22, KLST23, roth-sequential}.

Perhaps the most closely related work to ours is Zhao {\em et al.}'s 
\cite{zhao2021calibrating} work on decision calibration. They imagine a down-stream decision maker using the predictions to choose between a finite set of actions, subject to a loss function that depends only on the action and on the outcome. The predicted distribution should be indistinguishable from the true distribution in terms of the loss experienced by the decision maker (and this should hold for any such decision maker and any loss function). We view this as an expressive calibration notion: in particular, even if we only allow for two possible actions, decision calibration  (see Definition \ref{def:dec}) guarantees a sharp flavor of subset calibration: for any subset $T \subseteq [k]$ and any threshold $b \in [0,1]$, conditioning on instances where the predictor assigns total probability at least $b$ to the set $T$, the probability that the outcome lands in $T$ is at least $b$ (up to a small error).\footnote{The guarantee is even stronger: the conditional expectation of the predictions and the outcomes should be close.} They showed that this strong guarantee can be obtained using only $\poly(k)$ samples. We show, however, that the runtime complexity of obtaining decision calibration cannot be $\poly(k)$ (assuming the hardness of refuting random CSPs). Intuitively, the hardness is due to the ``sharpness'' of the guarantee: conditioning on the event that the probability of $T$ is {\em exactly} above the threshold $b$. This has the flavor of a halfspace learning guarantee, and this underlies our intractability result. In contrast, our notion of projected smooth calibration (and our results on sigmoids) enforces a ``softer'' Lipschitz condition, which makes the problem computationally tractable and allows us to construct efficient algorithms. 
On a more technical level, Zhao {\em et al.} \cite{zhao2021calibrating} require solving an optimization problem over the class of halfspaces. Noting that the objective is not differentiable, they present a heuristic gradient-based algorithm after relaxing the hard halfspace threshold using a differentiable sigmoid function. 
They allow the Lipschitz constant of the sigmoid function to grow arbitrarily large in order to recover the halfspaces in the limit.
However, they do not provide a provable guarantee on the correctness or efficiency of their algorithm. 
We show that such guarantees are unlikely to be established due to inherent intractability of the problem %
(see above and in \Cref{sec:hardness}).

In their work on outcome indistinguishability ``beyond Bernoulli'', Dwork {\em et al.} \cite{dwork2022beyond} also study meaningful predictions over non-Boolean outcome spaces. Their notion of Generative OI guarantees indistinguishability for a rich class of distinguishers that can examine the prediction and also features of the particular instance. This is quite expressive, and in particular, by formulating an appropriate class of distinguishers, their framework can capture all notions considered in this work. Their most general algorithm, for an outcome space of size $k$ and a (finite) class of distinguishers ${\cal A}$, requires sample complexity that is logarithmic in $k$ and in $|{\cal A}|$. The runtime is at least linear in the number of distinguishers $|{\cal A}|$. Guaranteeing subset calibration would require (at least) $\exp(k)$ distinguishers, so while their algorithm would be sample-efficient, its runtime is exponential in $k$.

The work of \cite{GopalanKSZ22} formulated the general notion of weighted calibration that we use. Their focus is on a particular instantiation of this notion they call low degree calibration, where the weight family is $P(d,1)^k$, where $P(d,1)$ contains all degree $d$ polynomials in $\bv$ with absolute values of coefficients summing to $1$. They do not consider the downstream calibration guarantees for binary classification tasks, rather their focus is on multicalibration and multigroup fairness. They present an auditing algorithm that runs in time $O(k^d)$. We show that by using kernel methods, one can obtain an auditor with running time $\poly(k, d)$ (Lemma \ref{lem:poly-weak}). 

Recently,  \cite{KLST23} and \cite{roth-sequential} studied relaxations of canonical calibration and gave algorithms for achieving them in the online setting. Similar to our work, they were motivated by giving meaningful guarantees for downstream tasks while avoiding the inefficiency inherent in canonical calibration.  However, their goal is to make calibrated predictions, which is challenging in the online setting, but becomes trivial in our offline setting (the constant predictor that always outputs the expectation of $\mD$'s outcomes is calibrated). 
Therefore, we focus instead on the auditing task of post-processing a given predictor. This auditing task is also considered in \cite{roth-sequential}, which gives online algorithms with running time growing polynomially with the size of the family of weight functions. 
We study the offline setting, where achieving running time that is linear in the size of the family of weight functions follows from \cite{GopalanKSZ22}, and our focus is on achieving polynomial running time even when the family of weight functions is exponential or infinite.

As in the standard setup of calibration, each prediction is a probability distribution over the possible labels. This distribution conveys the uncertainty of the predictor about the true label, and calibration can be viewed as a guarantee of accurate uncertainty quantification. Another common method for uncertainty quantification is conformal prediction (see e.g.\ \cite{conformal,gentle-conformal}), where the predictor outputs \emph{prediction sets} (sets of labels) aiming to provide a \emph{coverage guarantee}: the true label belongs to the prediction set with a certain, pre-specified probability. A recent line of work applies techniques from multicalibration to get robust conformal prediction algorithms that give coverage guarantees that hold not just on average, but \emph{conditionally} on every important subpopulation and beyond \cite{gupta-conformal,practical-conformal,batch-conformal}.

\subsection{Organization}
\label{sec:organization}
The rest of the paper is organized as follows. We start by defining old and new notions of multi-class calibration and discussing the connections among them in \Cref{sec:defs}. We prove an exponential sample complexity lower bound for canonical calibration in \Cref{sec:sample}. We define the auditing task and show a tight connection to agnostic learning in \Cref{sec:audit}. We apply the connection to show hardness of auditing for decision calibration and halfspaces in \Cref{sec:hardness}. We describe a general kernel method for auditing in \Cref{sec:kernel} and apply it to give efficient auditors for projected smooth calibration and sigmoid calibration in \Cref{sec:smooth}. 
We show barriers to further improving the efficiency of our algorithms by proving additional computational lower bounds in \Cref{sec:psmooth-lower}. 

%% file: new-defs.tex
\section{Multi-Class Calibration}
\label{sec:defs}
In this section, we discuss prior notions of multi-class calibration as well as their relationships, strengths, and drawbacks. 
We show that prior notions lack either expressivity or efficiency, and we introduce new notions to achieve a better balance between the two desiderata.

For a classification task with $k$ categories, we use $\mE_k = \{\e_1,\ldots,\e_k\}$ to denote the set of one-hot encodings of the categories. Here each $\e_i$ is the unit vector in $\R^k$ with the $i$-th coordinate being $1$. 

Throughout the paper, we use boldface letters to represent vectors in $\R^k$. For a vector $\bv\in \R^k$, we use $\bv\sps j\in \R$ to denote its $j$-th coordinate for every $j = 1,\ldots,k$.
We use $\Delta_k$ to denote the set of all vectors $\bv\in \R^k$ such that $\bv\sps j \ge 0$ for every $j = 1,\ldots,k$ and $\bv\sps 1 + \cdots + \bv\sps k = 1$.

For a set $\X$ of individuals, a predictor is a function $p:\X\to \Delta_k$ that assigns every individual $x\in \X$ a prediction vector $\bv = p(x)\in \Delta_k$, where each coordinate $\bv\sps j$ is the predicted probability that the label of $x$ falls in the $j$-th category.

\paragraph{Canonical Calibration.} For a ground-truth distribution $\mD_0$ of labeled examples $(x,\y)\in \X\times \mE_k$, we say a predictor $p:\X\to \Delta_k$ satisfies (perfect) \emph{canonical calibration} if
\[
\E_{(x,\y)\sim \mD_0}[\y|p(x) = \bv] = \bv \quad \text{for every }\bv\in \Delta_k.
\]
A simple but important observation is that the above definition only depends on the distribution of $(p(x),\y)\in \Delta_k\times \mE_k$. As a consequence, we obtain the following simplified but equivalent definition:
\begin{definition}
We say a distribution $\mD$ of $(\bv,\y)\in \Delta_k\times \mE_k$ satisfies (perfect) \emph{canonical calibration} if
\[
\E_{(\bv,\y)\sim \mD}[\y|\bv] =\bv.
\]
\end{definition}
In the definition above, we work with a distribution $\mD$ of $(\bv,\y)\in \Delta_k$  without explicitly stating that it is the distribution of $(p(x),\y)$ where $(x,\y)$ comes from the ground-truth distribution $\mD_0$, and $p$ is a predictor. We will use this convention throughout the paper.

It is folklore that the sample complexity of determining whether a distribution $\mD$ satisfies perfect canonical calibration grows exponentially in $k$. In \Cref{sec:sample}, we prove a stronger result (\Cref{thm:canonical}), showing that distinguishing whether a distribution $\mD$ satisfies perfect canonical calibration or it is $\Omega(1)$-far from canonical calibration (in $\ell_1$ distance) requires sample complexity exponential in $k$.

\paragraph{Weighted Calibration.} Due to the sample inefficiency of canonical calibration, many previous works considered relaxations of canonical calibration such as confidence calibration and top-label calibration. 
These notions can be framed as special cases of a general notion called \emph{weighted calibration} studied in \citet{GopalanKSZ22}:

\begin{definition}[Weighted calibration \citep{GopalanKSZ22}]
\label{def:weighted}
Let $\mW:\Dk \to [-1,1]^k$ be a family of weight functions.  We define the $\mW$-calibration error of a distribution $\mD$ of $(\bv,\y)\in \Delta_k\times \mE_k$ as 
\[ \CE_\mW(\mD) = \sup_{w \in \mW} \lt| \E_{(\bv,\y)\sim \mD} [\ip{\y - \bv}{w(\bv)}] \rt|. \]
We say that  $\mD$ is $(\mW, \alpha)$-calibrated if $\CE_{\mW}(\mD) \leq \alpha$. 
\footnote{The original definition in \cite{GopalanKSZ22} was presented in the more general context of multicalibration. Their definition allows for weight families whose range is $[0,1]^k$ rather than $[-1,1]^k$, but this is a technical issue.  }
\end{definition}
If a distribution $\mD$ satisfies perfect canonical calibration, then it is $(\mW,0)$-calibrated for any class $\mW$. When the class $\mW$ consists of all functions $w:\Dk\to [-1,1]^k$, $(\mW,0)$-calibration becomes equivalent to perfect canonical calibration.

\paragraph{Class-wise, Confidence, and Top-label Calibration.} 
The notion of weighted calibration is very general. By choosing the class $\mW$ appropriately, it recovers many concrete notions of calibration. 

Class-wise calibration \cite{kull2019beyond} is the following requirement:
\[
\E_{(\bv,\y)\sim\mD}[\y\sps \ell|\bv\sps \ell] = \bv\sps\ell \quad \text{for every}\ \ell = 1,\ldots,k.
\]
This is equivalent to $(\mW,0)$-calibration where $\mW$ consists of all functions $w$ mapping $\bv\in \Dk$ to $w(\bv) = \phi(\bv\sps \ell)\e_\ell$ for every $\phi:[0,1]\to [-1,1]$ and $\ell = 1,\ldots,k$.

Confidence calibration \cite{guo2017calibration} is also a special case of weighted calibration. For any $\bv\in \Dk$, let $\ell_\bv$ denote the coordinate $\ell\in \{1,\ldots,k\}$ that maximizes $\bv\sps \ell$. Confidence calibration is the following requirement:
\[
\E_{(\bv,\y)\sim\mD}[\y\sps {\ell_\bv}|\bv\sps {\ell_\bv}] = \bv\sps{\ell_\bv}. 
\]
This is equivalent to $(\mW,0)$-calibration where $\mW$ consists of all functions $w$ mapping $\bv\in \Dk$ to $w(\bv) = \phi(\bv\sps {\ell_\bv})\e_{\ell_\bv}$ for every  $\phi:[0,1]\to [-1,1]$.

Top-label calibration \citep{top-cal} is defined to be the following requirement:
\[
\E_{(\bv,\y)\sim\mD}[\y\sps {\ell_\bv}|\bv\sps {\ell_\bv},\ell_\bv] = \bv\sps{\ell_\bv}. 
\]
This is equivalent to $(\mW,0)$-calibration where $\mW$ consists of all functions $w$ mapping $\bv\in \Dk$ to $w(\bv) = \phi(\bv\sps {\ell_\bv},\ell_\bv)\e_{\ell_\bv}$ for every $\phi:[0,1]\times \{1,\ldots,k\}\to [-1,1]$.

\paragraph{Decision Calibration.} A drawback of class-wise, confidence, and top-label calibration is that they do not imply good calibration performance if the predictions are used for downstream tasks. To improve the expressivity while avoiding the exponential sample complexity of canonical calibration, \cite{zhao2021calibrating} introduced the notion of \emph{decision calibration}, where they studied downstream loss-minimization tasks of deciding which action to choose among a fixed set of actions based on the predictions. We focus on the special case of two actions, where the definition of decision calibration is as follows:

\begin{definition}[Decision Calibration \citep{zhao2021calibrating}]
\label{def:dec}
For a distribution $\mD$ of $(\bv,\y)\in \Delta_k\times \mE_k$, the decision calibration error of $\mD$ is defined to be
\footnote{
The original definition of decision calibration in \cite{zhao2021calibrating} takes a slightly different form:
\[
\dec(\mD):=\sup_{\br,\br'\in \R^k, b\in \R}\|(\y - \bv)\one(\langle \br, \bv\rangle > \langle \br',\bv \rangle)\|_2 + \|(\y - \bv)\one(\langle \br, \bv\rangle \le \langle \br',\bv \rangle)\|_2 .
\]
Here $\langle \br,\bv\rangle$ (resp.\ $\langle \br',\bv\rangle$) is the expected \emph{loss} of taking action $1$ (resp.\ action $2$) over the randomness in an outcome $\hat\y\in \Ek$ distributed with mean $\bv$. In fact \Cref{def:dec} is equivalent to this definition.
For any $\bv\in \Delta_k$, the sum of the coordinates of $\bv$ is $1$, i.e., $\langle \mathbf 1,\bv\rangle = 1$, where $\mathbf 1$ is the all-ones vector. Therefore, in \Cref{def:dec}, we have $\one(\langle \ba,\bv\rangle > b) = \one(\langle \ba - b\mathbf 1,\bv\rangle > 0)$, and thus restricting $b = 0$ does not change \Cref{def:dec}. Under this restriction, the equivalence between the two definitions follows by taking $\ba = \br - \br'$.
}
\[
\dec(\mD):=\sup_{\ba\in \R^k, b\in \R}\|(\y - \bv)\one(\langle \ba, \bv\rangle > b)\|_2 + \|(\y - \bv)\one(\langle \ba, \bv\rangle \le b)\|_2 .
\]
Equivalently, this is the weighted calibration error $\CE_\mW(\mD)$, where $\mW$ consists of functions mapping $\bv$ to $\one(\langle \ba, \bv\rangle > b)\bg + \one(\langle \ba, \bv\rangle > b)\bg'\in \R^k$ for every $\ba\in \R^k, b\in \R$, and $\bg,\bg'\in \R^k$ satisfying $\|\bg\|_2 \le 1$, $\|\bg'\|_2 \le 1$.
\end{definition}

The work of \cite{zhao2021calibrating} showed that decision calibration ensures a desirable indistinguishability property for downstream loss minimization tasks, demonstrating the expressivity of the notion.
However, we show that decision calibration is a computationally inefficient notion. Here, efficiency is evaluated on the \emph{auditing} task (defined formally in \Cref{sec:audit}), where the goal is to recalibrate a given mis-calibrated predictor while reducing its squared loss. \citep{zhao2021calibrating} showed that auditing for decision calibration has \emph{sample} complexity polynomial in $k$, improving over the exponential sample complexity of canonical calibration, but they fell short of proving a computational efficiency guarantee. Instead, they provided a heuristic algorithm for the auditing task without correctness or running time analyses. Our computational hardness results in \Cref{sec:hardness} show that under standard complexity-theoretic assumptions, there is no $\poly(k)$-time algorithm for auditing decision calibration.

\paragraph{Smooth Calibration.} We now introduce new notions of multi-class calibration inspired by a recent theory of \citet{utc1} on calibration measures in the binary setting.

Consider the downstream binary prediction task of predicting whether the true label belongs to a set $T \subset [k]$ of labels. 
Let $\ba:=\ind{T} \in \zo^k$ denote the indicator of the subset $T$. Given the distribution $\mD$ of $(\bv, \y)$, the predicted probability of this event is $\ip{\ind{T}}{\bv}$ while the true label is given by $\ip{\ind{T}}{\y}$. 
A natural approach to defining calibration notions for the multi-class setting is to seek good calibration guarantees for every such binary prediction tasks.

A well studied notion of calibration for binary classification is the notion of smooth calibration \cite{kakadeF08, FosterH18}. A key advantage of this notion is that it is robust to perturbations of the predictor, unlike notions such as $\ECE$. More recently, it plays a central role in the the work of \cite{utc1} and their theory of consistent calibration measures for binary classification. At a high level, these are calibration measures that are polynomially related to the (earthmover) distance to the closest perfectly calibrated predictor. 
\eat{
Meaningful calibration measures in the binary setting have been studied in the recent work \cite{utc1}.
In \citet{utc1}, a calibration measure is called a \emph{consistent calibration measure} if it is polynomially related to the \emph{distance to calibration}.\parik{not defined yet. Are we just trying to motivate smooth calibration?} A representative consistent calibration measure is the \emph{smooth calibration error} \cite{kakadeF08, utc1}.}
Applying the notion of smooth calibration error to the downstream binary prediction tasks for subsets $T \subseteq [k]$, we get the following definition:

\begin{definition}[Subset Smooth Calibration]
\label{def:sm-subset}
Let $\Lip$ be the class of $1$-Lipschitz functions $\phi:\R\to [-1,1]$.
For a distribution $\mD$ of $(\bv,\y)\in \Delta_k\times \mE_k$, we define the smooth calibration error of $\mD$ on the subset $T$ to be
\begin{align*}
\smCE_T(\mD) & = \sup_{\phi \in \Lip} \E_{\mD} [(\langle \ind T, \y\rangle - \langle \ind T, \bv  \rangle)\phi(\ip{\ind{T}}{\bv})]\\
& =  \sup_{\phi \in \Lip} \E_{\mD} [\ip{\ind{T}}{\y - \bv}\phi(\ip{\ind{T}}{\bv})].  
\end{align*}
We define the subset smooth calibration error of $\mD$ as
\[ \ssCE(\mD) =  \sup_{T \subseteq [k]} \smCE_T(\mD).  \]
More generally, for $m \ge 0$, we define the $m$-subset smooth calibration error of $\mD$ as
\[ \ssCE_m(\mD) =  \sup_{T \subseteq [k], |T| \le m} \smCE_T(\mD).  \]
 \end{definition}
Note that we can define subset smooth calibration as a special case of weighted calibration. We define $\mW_{\sset m}$ to be the set of all functions $w:\Delta_k\to [-1,1]^k$ such that there exist $T\subseteq [k]$ and $\phi\in \Lip$ satisfying $|T| \le m$ and $w(\bv) = \ind T\phi(\langle \ind T,\bv \rangle)$ for every $\bv\in \Delta_k$. Then
\[ \ssCE_m(\mD) =  \CE_{\mW_{\sset m}}(\mD).  \]

In the binary setting, a result of \cite{utc1} shows that the smooth calibration error is polynomially related to the (earthmover) distance to the nearest perfectly calibrated predictor. 
Therefore, a small subset smooth calibration error in our multi-class setting implies that for every subset $T\subseteq [k]$, the prediction $\langle \ind T, \bv\rangle$ is close to perfect calibration for the corresponding downstream binary prediction task.

Having demonstrated the expressivity of subset smooth calibration, we move on to establish its efficiency.
The main algorithmic result of our paper is that auditing for subset smooth calibration can be achieved in time polynomial in $k$ (for any fixed error parameter $\alpha$, see \Cref{sec:audit} for formal definition of auditing).
That is, subset smooth calibration simultaneously achieves strong expressivity and computational efficiency.
In fact, the efficiency of our auditing algorithm extends to a more expressive notion which we call \emph{projected smooth calibration}, where we generalize indicators of sets that are vectors in $\zo^k$ to allow vectors in $[-1,1]^k$. 

\begin{definition}[Projected Smooth Calibration]
\label{def:p-smooth}
For $m \ge 0$, let $\mH_{m\text{-}\pLip}$ denote the set of all functions $h:\Dk \to [-1,1]$ such that there exist $\phi \in \Lip$ and $\ba\in [-1,1]^k$ with $\|\ba\|_2^2\le m$ satisfying
\[ h(\bv) = \phi(\ip{\ba}{\bv}) \quad \text{for every } \bv\in \Delta_k.\] 
Define the $m$-projected smooth calibration error as
\[ \psCE_m(\mD) = \CE_{\hplip m^k}(\mD). \]
\end{definition}

In measuring $\psCE$, we audit each coordinate $i \in [k]$ using a distinct function $h\sps i \in \hplip m$. 
We also consider a further strengthening of projected smooth calibration by allowing  arbitrary $\ell_1$-Lipshcitz functions in each coordinate:
\begin{definition}[Full Smooth Calibration]
    Let $\hflip$ denote the set of all functions $h: \Dk \to [-1,1]$ such that 
    \[ |h(\bv) - h(\bv')| \leq \norm{\bv - \bv'}_1 \quad \text{for every }\bv, \bv' \in \Dk. \] 
    Define the full smooth calibration error of a distribution $\mD$ of $(\bv,\y)\in \Delta_k\times \mE_k$ as
    \[ \fsCE(\mD) = \CE_{\hflip ^k}(\mD) .\]
\end{definition}

\begin{lemma}
\label{lm:smooth}
    For any $m \ge 0$, for any distribution $\mD$ of $(\bv,\y)\in \Delta_k\times \mE_k$,
    \[\ssCE_m(\mD) \leq \psCE_m(\mD) \leq \fsCE(\mD). \]
\end{lemma}
\begin{proof}
    To prove the first inequality, let $T \subset [k]$ be the set of size bounded by $m$ that maximizes $\smCE_T(\mD)$, and $\phi \in \Lip$ the Lipschitz function that witnesses it, so that
    \[ \ssCE_m(\mD) = \E_{\mD} [\ip{\ind{T}}{\y - \bv}\phi(\ip{\ind{T}}{\bv})] = \E_{\mD} [\ip{\phi(\ip{\ind{T}}{\bv})\ind{T}}{\y - \bv}]. \]
    We define the auditor function $w \in \hplip m^k$ where 
    \begin{align*} 
        w\sps i(\bv) = \ind{T}\phi(\ip{\ind{T}}{\bv}) = \begin{cases} \phi(\ip{\ind{T}}{\bv}) \ \text{for}\  i \in T\\
        0 \ \text{otherwise}
        \end{cases}
    \end{align*}
    Hence
    \begin{align*} 
    \psCE_m(\mD) &= \max_{w' \in \hplip m^k} \E_\mD[\ip{\y - \bv}{w'(\bv)}] \\
    & \geq \E_\mD[\ip{\y - \bv}{w(\bv)}]\\
    &= \ip{\y - \bv}{\ind{T}}\phi(\ip{\ind{T}}{v})\\
    &= \ssCE_m(\mD).
    \end{align*}
    
    The second inequality is implied by the inclusion $\hplip m \subseteq \hflip$. To prove this inclusion, note that for any function $h \in \hplip m$, there exists $\phi \in \Lip, \ba \in [-1,1]^k$ such that $h(\bv) = \phi(\ip{\ba}{\bv})$ for every $\bv\in \Dk$. We have
    \begin{align*}
    |h(\bv) - h(\bv')| & =
        |\phi(\ip{\ba}{\bv}) - \phi(\ip{\ba}{\bv'})|\\ &\leq |\ip{\ba}{\bv} - \ip{\ba}{\bv'}|\\
        & = |\ip{\ba}{\bv - \bv'}|\\
        & \leq \norm{\ba}_\infty\norm{\bv - \bv'}_1\\
        & \leq \norm{\bv -\bv'}_1
    \end{align*}
    where the first inequality uses the Lipschitz property of $\phi$. This shows $h \in \hflip$, which completes the proof.
\end{proof}

In \Cref{sec:smooth} we show that both subset smooth calibration and projected smooth calibration allow efficient auditing, whereas in \Cref{thm:full} we show that full smooth calibration requires sample complexity exponential in $k$.

\input{./sample-lower}

\section{Auditing for Weighted Calibration and Agnostic Learning}
\label{sec:audit}

In this section, we study the sample and computational complexity of weighted calibration (\Cref{def:weighted}), where the complexity is measured in an \emph{auditing} task we define below. Specifically, for any weight family $\mW$, we show an equivalence between the auditing task and the well-studied agnostic learning task in the learning theory literature. This equivalence allows us to establish both computational lower bounds and efficient algorithms for specific weight families $\mW$ in \Cref{sec:hardness,sec:kernel,sec:smooth,sec:psmooth-lower}.

\paragraph{Auditing for weighted calibration.} The notion of weighted calibration gives rise to a natural decision problem, which we call the decision version of auditing calibration: given a predictor $p$, can we decide whether or not it is $(\mW,\alpha)$ calibrated? In the event that $p$ is not calibrated,  we would ideally like to post-process its predictions to get a new predictor $\kappa(p)$ for $\kappa:\Dk \to \Dk$, so that $\kappa(p)$ is $(\mW,\alpha)$-calibrated.  This post-processing goal needs to be formulated carefully, since one can always get perfect calibration using a trivial predictor that constantly predicts $\E[\y]$. A natural formulation that avoids such trivial solutions is to require that the post-processing does not harm some  measure of accuracy such as the expected squared loss of $p$. 

One can achieve both these goals by solving a search problem which we call \emph{auditing with a witness} defined below.

\begin{definition}[Auditing with a witness]
\label{def:audit-guarantee}
    An $(\alpha, \beta)$ auditor for $\mW$ is an algorithm that when given access to a distribution $\mD$ where $\CE_\mW(\mD) > \alpha$ returns a function $w':\Delta_k\to [-1,1]^k$ (which need not belong to $\mW$) such that
\begin{equation}
\label{eq:audit-guarantee}
   \E_{(\bv,\y)\sim \mD} [\ip{\y - \bv}{w'(\bv)}] \geq \beta. 
\end{equation}
Concretely, the auditor takes i.i.d.\ examples $(\bv_1,\y_1),\ldots,(\bv_n,\y_n)$ drawn from $\mD$, and the output function $w'$ should satisfy the inequality above with probability at least $1-\delta$ over randomness in the examples and the auditor itself, where $\delta\in (0,1/3)$ is the \emph{failure probability parameter}.
\end{definition}

As demonstrated in previous work (for instance \cite{hkrr2018, GopalanKSZ22}), a solution to this search problem allows us to solve both the decision problem of auditing for calibration, and in the case when $p$ is not $(\mW, \alpha)$-calibrated, we can use the witness to post-process $p$ and produce a predictor $\kappa(p)$ with lower squared loss, that is $(\mW, \alpha)$-calibrated.

\begin{lemma}
\label{lem:audit}
Given a predictor $p:\X \to \Dk$ and access to an $(\alpha, \beta)$-auditor for $\mW$, there is an algorithm that computes a post-processing function $\kappa: \Dk \to \Dk$ so that $\kappa(p)$ is $(\mW, \alpha)$-calibrated and its squared loss is not larger than that of $p$. The algorithm uses at most $O(k/\beta^2)$ calls to the $(\alpha, \beta)$-auditor.
\end{lemma}
\begin{proof}
We start off with $p_0 = p$. If $p$ is not $(\mW, \alpha)$-calibrated, then the auditor produces $w'$ satisfying Equation \eqref{eq:audit-guarantee}. Following the proof of \cite[Lemma 33]{GopalanKSZ22}, we can now update $p$ to $\kappa(p)$ using $w'$ so that we get a decrease in the expected squared loss:
\[
\E[\|\kappa(\bv)- \y\|_2^2] \le \E[\|\bv- \y\|_2^2] - \Omega(\beta^2/k).
\]

Note that the squared loss is bounded in the interval $[0,4]$ because $\|\bv - \y\|_2 \le \|\bv\|_2 + \|\y\|_2 \le \|\bv\|_1 + \|\y\|_1 \le 2$.
Thus by repeatedly using the auditor and applying the update at most $O(k/\beta^2)$ times, we can eventually achieve $(\mW,\alpha)$ calibration with decreased expected squared loss.
\end{proof}

We make some observations about the role that the different parameters $\alpha, \beta$ and $\mW$ play in the complexity of auditing with a witness. 
\begin{itemize}

\item Auditing becomes easier for smaller $\beta$.  The $\beta$ parameter affects the running time, but not the final calibration guarantee. Thus an $(\alpha, \beta/10)$ auditor will result in the same guarantee as an $(\alpha, \beta)$ auditor, but at the cost of more iterations. Since we are interested in the question of whether auditing can be done in time $\poly(k)$ versus $\exp(k)$, we do not optimize too much for $\beta$, and are fine with losing polynomial factors in it. 

\item In contrast, auditing gets harder for smaller $\alpha$, since the auditor is required to detect smaller violations of calibration. The final guarantee is also much more sensitive to $\alpha$: a $(2\alpha, \beta)$ auditor can only be used to guarantee $(\mW, 2\alpha)$ calibration, but not $(\mW, \alpha)$ calibration. 

\item The complexity of auditing increases as the the weight function family becomes larger. If $\mW_1 \subseteq \mW_2$, then an $(\alpha, \beta)$-auditor for $\mW_2$ is also an $(\alpha, \beta)$-auditor for $\mW_1$, since 
$\CE_{\mW_2}(\mD) \geq \CE_{\mW_1}(\mD)$ so the auditor is guaranteed to produce a witness whenever $p$ is not $(\mW_1, \alpha)$-calibrated. It might happen that $\CE_{\mW_1}(\mD) \leq \alpha$ whereas $\CE_{\mW_2}(\mD)> \alpha$. In such a scenario, an auditor for $\mW_2$ will still find a witness to miscalibration. This is not required by our definition of auditor for $\mW_1$, but it is allowed.
\end{itemize}

\eat{
if the $\beta$ parameter of the auditor becomes smaller, we may need more iterations but can still eventually achieve the same goal. However, if the $\alpha$ parameter becomes a larger value $\alpha'>\alpha$, an $(\alpha',\beta)$ auditor is in general no longer helpful for achieving $(\mW,\alpha)$ calibration with decreased expected loss, because the initial predictor may be $(\mW,\alpha')$ calibrated but not $(\mW,\alpha)$ calibrated. In that sense, the strength of the auditor is more sensitive to $\alpha$ than to $\beta$, and in our results we focus on optimizing the $\alpha$ parameter as long as $\beta$ takes a reasonable value (e.g.\ being polynomially small rather than exponentially small).}

%% file: sample-lower.tex
\section{Sample Complexity of Canonical Calibration}
\label{sec:sample}
The main goal of this section is to prove that distinguishing whether a distribution $\mD$ of $(\bv,\y)$ satisfies perfect canonical calibration or $\mD$ is far from canonical calibration requires sample complexity exponential in $k$.

We use the following definition of distance to canonical calibration, generalizing the \emph{lower distance to calibration} in \cite{utc1} from the binary setting to the multi-class setting.
\begin{definition}[Distance to Canonical Calibration]
Consider a distribution $\mD$ of $(\bv,\y)\in \Dk\times \Ek$. We define $\ext(\mD)$ to be the set of distributions $\Pi$ of $(\bu,\bv,\y)$ where the marginal distribution of $(\bv,\y)$ is $\mD$, and the marginal distribution of $(\bu,\y)$ satisfies perfect canonical calibration. We define the distance to calibration, denoted by $\dce(\mD)$, as follows:
\[
\dce(\mD) := \inf_{\Pi\in \ext(\mD)}\E_{\Pi}\|\bu - \bv\|_1.
\]
\end{definition}

Here is our sample complexity lower bound:

\begin{theorem}
\label{thm:canonical}
Let $A$ be an algorithm that takes examples $(\bv_1,\y_1),\ldots,(\bv_n,\y_n)\in \Dk\times \Ek$ drawn i.i.d.\ from a distribution $\mD$ as input, and outputs \acc or \rej. Assume that for any distribution $\mD$ satisfying perfect canonical calibration, algorithm $A$ outputs \acc\ with probability at least $2/3$. Also, for some $\alpha>0$, assume that for any distribution $\mD$ satisfying $\dce(\mD) \ge \alpha$, algorithm $A$ outputs \rej\ with probability at least $2/3$. Then for some absolute constants $k_0 > 0$ and $c > 0$, assuming $k\ge k_0$, we have $n \ge (c/\alpha)^{(k-1)/2}$.
\end{theorem}

To prove \Cref{thm:canonical}, we use the following lemma to connect the distance to canonical calibration $\dce(\mD)$ with the full smooth calibration error $\fsCE(\mD)$.
\begin{lemma}
\label{lm:full-dce}
For any distribution $\mD$ over $\Dk\times \Ek$,
$\fsCE(\mD) \le 4\dce(\mD)$.
\end{lemma}
\begin{proof}
Consider any function $w\in \hflip ^k$ and any distribution $\Pi\in \ext(\mD)$. 
By the definition of $\hflip^k$, for any $\bu,\bv\in \Dk$, we have
\begin{equation}
\label{eq:infty-lip}
\|w(\bu) - w(\bv)\|_\infty \le \|\bu - \bv\|_1.
\end{equation}
By the definition of $\ext(\mD)$, for $(\bu,\bv,\y)\sim \Pi$, the distribution of $(\bu,\y)$ satisfies perfect canonical calibration, and thus
\begin{equation}
\label{eq:Pi-calibrated}
\E_\Pi[\langle \y - \bu, w(\bu)\rangle] = 0.
\end{equation}
Therefore,
\begin{align*}
\E_\mD[\langle \bv - \y, w(\bv)\rangle] & \le |\E_\Pi[\langle \bv - \y, w(\bv) - w(\bu)\rangle] | + \E_\Pi[\langle \bv - \y, w(\bu)\rangle]\\
& \le 2\E_\Pi\|\bu - \bv\|_1 + \E_\Pi[\langle \bv - \y, w(\bu)\rangle]\tag{by \eqref{eq:infty-lip}}\\
& = 2\E_\Pi\|\bu - \bv\|_1 + \E_\Pi[\langle \bv - \bu, w(\bu)\rangle]\tag{by \eqref{eq:Pi-calibrated}}\\
& \le 4\E_\Pi\|\bu - \bv\|_1,
\end{align*}
where the last inequality holds because $\|w(\bu)\|_\infty\le 1$. The proof is completed by taking supremum over $w\in \hflip^k$ and infimum over $\Pi\in \ext(\mD)$.
\end{proof}
By \Cref{lm:full-dce}, \Cref{thm:canonical} is a direct corollary of the following theorem which gives a  sample complexity lower bound for distinguishing  perfect canonical calibration from having a large full smooth calibration error:
\begin{theorem}
\label{thm:full}
Let $A$ be an algorithm that takes examples $(\bv_1,\y_1),\ldots,(\bv_n,\y_n)\in \Dk\times \Ek$ drawn i.i.d.\ from a distribution $\mD$ as input, and outputs \acc or \rej. Assume that for any distribution $\mD$ satisfying perfect canonical calibration, algorithm $A$ outputs \acc\ with probability at least $2/3$. Also, for some $\alpha>0$, assume that for any distribution $\mD$ satisfying $\fsCE(\mD) \ge \alpha$, algorithm $A$ outputs \rej\ with probability at least $2/3$. Then for some absolute constants $k_0 > 0$ and $c > 0$, for all $k\ge k_0$ we have $n \ge (c/\alpha)^{(k-1)/2}$.
\end{theorem}

Our proof of \Cref{thm:full} starts with the following lemma which can be proved by a standard greedy algorithm:

\begin{lemma}
\label{lm:packing}
There exist absolute constants $c > 0$ and $k_0 > 0$ with the following property. For any positive integer $k > k_0$ and any $\varepsilon > 0$, there exists a set $V\subseteq\Delta_k$ with the following properties:
\begin{enumerate}
\item $|V| \ge (c/\varepsilon)^{k-1}$;
\item $\|\bv_1 - \bv_2\|_1 \ge \varepsilon$ for any distinct $\bv_1,\bv_2\in V$;
\item $\|\bv - \e_i\|_1 \ge 1/3$ for any $\bv\in V$ and $i\in \{1,\ldots,k\}$.
\end{enumerate}
\end{lemma}
\begin{proof}
The lemma can be proved by a simple greedy algorithm. Let us start with $V = \emptyset$ and repeat the following step: if there exists $\bv'\in \Delta_k$ such that $\|\bv' - \bv\|_1 \ge \varepsilon$ for every $u\in U$ and $\|\bv' - \e_i\|_1 \ge 1/3$ for every $i = 1,\ldots,k$, we add $\bv'$ to $V$. We repeat the step until no such $\bv'$ exists to obtain the final $V$. Clearly, $V$ satisfies properties 2 and 3 required by the lemma. It remains to prove that $V$ also satisfies property 1. 

Consider the final $V$ in the process of the algorithm. For any $\bv\in V$, consider a set $S_\bv$ consisting of all points $\bs\in \R^{k-1}$ such that $\|\bs - \bv|_{1,\ldots,k-1}\|_1 \le \varepsilon$. Similarly, for every $i = 1,\ldots,k$, consider a set $S_i$ consisting of all points $\bs\in \R^{k-1}$ such that $\|\bs - \e_i|_{1,\ldots,k-1}\|_1 \le 1/3$. Also, consider the set $S$ consisting of all points $\bs\in \R_{\ge 0}^{k-1}$ such that $\|\bs\|_1 \le 1$. If $S\setminus ((\bigcup_{\bv\in V}S_\bv)\cup (\bigcup_{i=1}^k S_i))$ is non-empty, then we can take any $\bs$ in that set and construct a vector $\bv' = (\bs\sps 1,\ldots,\bs\sps {k-1}, 1- \bs\sps 1 - \cdots - \bs\sps {k-1})\in \Delta_k$. Since $\bs\notin S_v$, it is easy to see that $\|\bv' - \bv\|_1 > \varepsilon$ for every $\bv\in V$. Similarly, $\|\bv' - \e_i\| > 1/3$ for every $i = 1,\ldots,k$. Therefore, the iterative steps of the algorithm can be continued. For the final $V$, it must hold that $S\setminus ((\bigcup_{\bv\in V}S_\bv)\cup (\bigcup_{i=1}^k S_i))$ is empty. The volume of each $S_\bv$ is $(2\varepsilon)^{k-1}$ times the volume of $S$, and the volume of each $S_i$ is $(2/3)^{k-1}$ times the volume of $S$. Therefore,
\[
(2\varepsilon)^{k-1}|V| + (2/3)^{k-1}k \ge 1.
\]
When $k$ is sufficiently large, we have $(2/3)^{k-1}k \le 1/2$, in which case $|V| \ge (1/2)(1/(2\varepsilon))^{k-1}\ge (c/\varepsilon)^{k-1}$, where the last inequality holds whenever $k$ is sufficiently large and $c > 0$ is sufficiently small.
\end{proof}
In the lemma below, we use the set $V$ from \Cref{lm:packing} to construct candidate distributions with large full smooth calibration error. Later in \Cref{lm:indistinguishable} we combine these distributions to achieve indistinguishability from a distribution with no calibration error, unless given at least $\exp(k)$ examples.
\begin{lemma}
\label{lm:miscalibrated}
For a sufficiently large positive integer $k$ and $\varepsilon \in (0,1/2)$, let $V\subseteq\Delta_k$ be the set guaranteed by \Cref{lm:packing}. For a function $w:V\to \mE_k$, define distribution $\mD_w$ of $(\bv,\y)\in V\times \Ek$ such that $\bv$ is distributed uniformly over $V$ and $\y = w(\bv)$. Then $\fsCE(\mD_w)\ge \varepsilon/12$.
\end{lemma}
\begin{proof}
For any $\bv\in V$, by property 3 in \Cref{lm:packing} and the fact that $w(\bv)\in \{\e_1,\ldots,\e_k\}$, we have $\|\bv - w(\bv)\|_1\ge 1/3$. 
Since $\bv\in \Dk$ and $w(\bv)\in \Ek$, we can separately consider the unique non-zero coordinate of $w(\bv)$ and the other zero coordinates to get
\[
1/3 \le \|\bv - w(\bv)\|_1 = (1 - \langle\bv,w(\bv) \rangle) + \langle \bv, \mathbf 1 - w(\bv) \rangle = 2(1 - \langle \bv, w(\bv)\rangle),
\]
where $\mathbf 1\in \R^k$ is the all-ones vector.
Therefore, $\langle w(\bv) - \bv,w(\bv)\rangle = (1 - \langle \bv, w(\bv)\rangle ) \ge 1/6$, and thus
\[
\E_{(\bv,\y)\sim \mD_w}[\langle \y - \bv,w(\bv)\rangle] \ge 1/6.
\]
To complete the proof, it remains to show that $w$ is $(2/\varepsilon)$-Lipschitz over $V$ (we can then extend $w$ to a $(2/\varepsilon)$-Lipschitz function over $\Dk$ by standard construction). For any distinct $\bv,\bv'\in V$, we have 
\[ \|w(\bv) - w(\bv')\|_\infty \le \|w(\bv) - w(\bv')\|_1 \le 2 \le (2/\varepsilon)\|\bv - \bv'\|_1,\] 
where the last inequality uses property 2 in \Cref{lm:packing}.
\end{proof}

\begin{lemma}
\label{lm:indistinguishable}
Let $A$ be any algorithm that takes $(\bv_1,\y_1),\ldots,(\bv_n,\y_n)\in V\times \mE_k$ as input, and outputs \acc\ or \rej. Let $p_1$ be the acceptance probability when we first draw $v_i$ independently and uniformly from $V$, and then draw each $\y_i$ independently with $\E[\y_i] = \bv_i$.
Let $p_2$ be the acceptance probability where we first draw $w:V\to \mE_k$ 
such that for every $\bv\in V$, $w(\bv)$ is distributed independently with mean $\bv$, and then draw each $(\bv_i,\y_i)$ independently from $\mD_w$. Then,
\[
|p_1 - p_2| \le O(n^2/|V|).
\]
\end{lemma}
\begin{proof}
Assume without loss of generality that $n < |V|$.
Let $p_3$ denote the acceptance probability when we first draw $\bv_1,\ldots,\bv_n$ uniformly from $V$ \emph{without replacement}, and then draw each $\y_i\in \Ek$ independently with mean $\bv_i$. We relate $p_1$ and $p_2$ to $p_3$ as follows.

Suppose we first draw each $\bv_i$ independently and uniformly from $V$, and then draw each $\y_i$ independently with $\E[\y_i] = \bv_i$. The probability that $\bv_1,\ldots,\bv_n$ are distinct is 
\[
p_4:=(1 - 1/|V|)\cdots (1 - (n-1)/|V|)\ge 1 - 1/|V| - \cdots - (n - 1)/|V| \ge 1 - O(n^2/|V|). \]
Conditioned on that event, the acceptance probability is exactly $p_3$. Conditioned on the complement of that event, the acceptance probability is bounded in $[0,1]$. Therefore, 
\[
p_3p_4 \le p_1\le p_3p_4 + (1 - p_4).
\]
Similarly, we can show that
\[
p_3p_4 \le p_2\le p_3p_4 + (1 - p_4).
\]
Combining these inequalities, we get $|p_1 - p_2|\le 1 - p_4 \le O(n^2/|V|)$.
\end{proof}

\begin{proof}[Proof of \Cref{thm:full}]
Consider the set $V$ from \Cref{lm:packing} where we choose $\varepsilon$ to be $12 \alpha$.
If $\mD$ is the distribution of $(\bv,\y)\in  V\times \Ek$ where $\bv$ is chosen uniformly at random from $V$ and $\E[\y|\bv] = \bv$, then algorithm $A$ outputs \acc\ with probability at least $2/3$. If $\mD$ is $\mD_w$ for some $w:V\to \mE_k$, then by \Cref{lm:miscalibrated}, algorithm $A$ outputs \acc\ with probability at most $1/3$.

By \Cref{lm:indistinguishable}, we have $n \ge \Omega(\sqrt{|V|})$. By Property 1 in \Cref{lm:packing}, we get $n \ge \Omega(\sqrt{|V|})\ge (c/\alpha)^{(k-1)/2}$ for a sufficiently small absolute constant $c > 0$ assuming $k$ is sufficiently large.
\end{proof}

%% file: new-equiv.tex
\paragraph{Agnostic learning.}
We understand the complexity of auditing for multiclass calibration by connecting it to the well-studied problem of agnostic learning in the standard binary classification setting.
For a distribution $\mU$ of $(\bv,z)\in \Dk\times [-1,1]$ and a class $\mH$ of functions $\Dk \to [-1,1]$, we define
\[
\opt(\mH,\mU):=\sup_{h\in \mH}|\E[h(\bv)z]|.
\]
\begin{definition}[Weak agnostic learner]\cite{SBD2, KalaiMV08}
    Let $\alpha \geq \beta \in [0,1]$. An $(\alpha, \beta)$ agnostic learner for $\mH$ is an algorithm that when given sample access to a distribution $\mU$ over $\Delta_k \times [-1,1]$ such that $\opt(\mH, \mU) \geq \alpha$ returns $h':\Delta_k\to [-1,1]$ such that 
    \[ \E_{(\bv, z) \sim \mU} [h'(\bv)z] \geq \beta. \]
More concretely, the learner takes i.i.d.\ examples $(\bv_1,z_1),\ldots,(\bv_n,z_n)$ drawn from $\mU$, and the output function $h'$ should satisfy the inequality above with probability at least $1-\delta$ over randomness in the examples and the learner itself, where $\delta\in (0,1/3)$ is the \emph{failure probability parameter}.
\end{definition}

Similarly to auditing, the strength of an agnostic learner is more sensitive to the $\alpha$ parameter than the $\beta$ parameter. Known results on agnostic boosting \cite{KalaiMV08, feldman2009distribution, kk09} show that the existence of an $(\alpha, \beta)$-weak agnostic learner implies the existence of an strong agnostic learner with polynomially increased time and sample complexity depending on $1/\beta$ (see the citations for a precise statement).

In the rest of the section we present our main result connecting the agnostic learning task for a class $\mH$ and the auditing task for $\mH^k$. 
\subsection{Auditing from Agnostic Learning}
\label{sec:red-algo}

\begin{theorem}
\label{thm:red}
Given an $(\alpha/3,\beta)$ weak agnostic learner for $\mH$ with sample complexity $n_0$, running time $T_0$ and failure probability parameter $\delta/2$, we can construct an $(\alpha,\alpha\beta/6k)$ auditor for $\mH^k$ with sample complexity $n = O(kn_0/\alpha + k^2\alpha^{-2}\beta^{-2}\log(k/\delta))$, time complexity $O(kT_0 + kn)$, and failure probability parameter $\delta$.
\end{theorem}
A natural idea for proving the theorem above is to apply the agnostic learner on each coordinate of the residual $\z:= \y - \bv$ in the auditing task. Specifically, in the auditing task, we assume
\[
\E[\langle \z, w(\bv)\rangle]  = \E[\langle \y - \bv, w(\bv)\rangle]> \alpha
\]
for some $w\in \mH^k$. Expressing $w(\bv)$ as $(w\sps 1(\bv),\ldots,w\sps k(\bv))$ where each $w\sps j\in \mH$, we have
\[
\sum_{j=1}^k\E[\z\sps j w\sps j(\bv)] > \alpha,
\]
which implies that there exists $j\in \{1,\ldots,k\}$ such that 
\begin{equation}
\label{eq:coordinate-cor}
\E[\z\sps j w\sps j(\bv)] > \alpha/k.
\end{equation}
If we only use \eqref{eq:coordinate-cor}, we would need an $(\alpha/k,\beta)$ agnostic learner to prove \Cref{thm:red}, but we only have an $(\alpha/3,\beta)$ agnostic learner.

To avoid the loss of a factor of $k$, we define $\z$ in a better way that leverages the fact that $\y,\bv\in \Delta_k$. Specifically, we note that the vector $\frac 12(\y - \bv)$ has $\ell_1$ norm at most $1$, and thus it is the mean of a distribution over $\mE_k\cup (-\mE_k)$. Given $\y$ and $\bv$, we draw $\z$ randomly from that distribution. We have
\[
\E[\langle \z, w(\bv)\rangle]  = \E\left[\left\langle (\y - \bv)/2, w(\bv)\right\rangle\right]> \alpha/2.
\]
Given $\z\in \mE_k\cup(-\mE_k)$, we use $\ell_\z \in [k]$ to denote the unique index such that $\z\sps {\ell_\z}\ne 0$. We have $\langle \z, w(\bv)\rangle = \z\sps {\ell_\z} w\sps {\ell_\z}(\bv)$ and thus
\[
\E[\z\sps {\ell_\z} w\sps {\ell_\z}(\bv)] > \alpha/2.
\]
Therefore, there exists $j\in \{1,\ldots,k\}$ such that
\[
\E[\z\sps {j} w\sps {j}(\bv)|\ell_\z = j] > \alpha/2 > \alpha/3.
\]
This allows us to use an $(\alpha/3,\beta)$ agnostic learner. 

\begin{proof}[Proof of \Cref{thm:red}]
In the auditing task, we assume that the input data points $(\bv_i,\y_i)$ are drawn i.i.d.\ from a distribution $\mD$ satisfying
\begin{equation}
\label{eq:red-0}
\E_{(\bv,\y)\sim \mD}[\langle \y - \bv, w(\bv)\rangle] > \alpha \quad\text{for some }w\in \mH^k.
\end{equation}
Given $(\bv,\y)$ drawn from $\mD$, we draw $\z$ randomly from $\mE_k\cup(-\mE_k)$ such that $\E[\z|\bv,\y] = (\y - \bv)/2$. This is possible because $\|\y - \bv\|_1 \le 2$. A concrete way to draw $\z$ is the following. With probability $1/2$, we set $\z$ to be $\y\in \mE_k$, and with the remaining probability $1/2$, we draw $\z$ randomly from $-\mE_k$ with expectation $-\bv$.

Given $\z\in \mE_k\cup(-\mE_k)$, we define a random variable $\ell_\z\in \{1,\ldots,k\}$ such that $\ell_\z$ is the unique index satisfying $\z\sps {\ell_\z} \ne 0$. 
For any $w\in \mH^k$, there exists $w\sps 1,\ldots,w\sps k\in \mH$ such that $w(\bv) = (w\sps 1(\bv),\ldots,w\sps k(\bv))$ for every $\bv\in \Delta_k$.
We have $\langle \z, w(\bv)\rangle = \z\sps {\ell_\z} w\sps {\ell_\z}(\bv)$ and thus \eqref{eq:red-0} implies
\[
\E[\z\sps {\ell_\z} w\sps {\ell_\z}(\bv)\rangle] = \E[\langle \z, w(\bv)\rangle]  = \E\left[\left\langle (\y - \bv)/2, w(\bv)\right\rangle\right]> \alpha/2.
\]
Let $\mU_j$ denote the conditional distribution of $(\bv, \z\sps j)\in \Delta_k\times \mE_k$ given $\ell_\z = j$. We have
\begin{equation}
\label{eq:red-1}
\sum_{j=1}^k\Pr[\ell_\z = j]\E_{(\bv,z)\sim \mU_j}[z w\sps j(\bv)] > \alpha/2.
\end{equation}
Now we show that there exists $j\in \{1,\ldots,k\}$ such that $\Pr[\ell_\z = j] \ge \alpha/6k$ and $\E_{(\bv,z)\sim \mU_j}[z w\sps j(\bv)] > \alpha/3$. If this is not the case, then
\begin{align*}
&\sum_{j=1}^k\Pr[\ell_\z = j]\E_{(\bv,z)\sim \mU_j}[z w\sps j(\bv)]\\
= {} & \sum_{j:\Pr[\ell_\z = j] < \alpha/6k}\Pr[\ell_\z = j]\E_{(\bv,z)\sim \mU_j}[z w\sps j(\bv)] + \sum_{j:\Pr[\ell_\z = j] \ge \alpha/6k}\Pr[\ell_\z = j]\E_{(\bv,z)\sim \mU_j}[z w\sps j(\bv)]\\
\le {} & \alpha/6 + \alpha/3\\
= {} & \alpha/2,
\end{align*}
giving a contradiction with \eqref{eq:red-1}.

We have shown that there exists $j^*\in \{1,\ldots,k\}$ and $h\in \mH$ such that $\Pr[\ell_\z = j^*] \ge \alpha/6k$ and $\E_{(z,\bv)\sim \mU_{j^*}}[z h(\bv)] > \alpha/3$. To solve the auditing task given examples $(\bv_1,\y_1),\ldots,(\bv_n,\y_n)$, we first draw $\z_1,\ldots,\z_n\in \mE_k\cup(-\mE_k)$ independently such that $\E[\z_i|\bv_i,\y_i] = \y_i - \bv_i$. Now $(\bv_1,\y_1,\z_1),\ldots,(\bv_n,\y_n,\z_n)$ are distributed independently from the joint distribution of $(\bv,\y,\z)$.
For every $j$, we define $I_j:=\{i\in \{1,\ldots,n\}:\ell_{\z_i} = j\}$. If $|I_j| \ge n_0$, we run the agnostic learner on the data points $((\bv_i,\z_i\sps j))_{i\in I_j}$ to obtain a function $h\sps j:\Delta_k\to [-1,1]$. We define $w_j:\Delta_k\to [-1,1]^k$ such that $(w_j(\bv))\sps {j'} = 0$ if $j'\ne j$ and $(w_j(\bv))\sps {j'} = h\sps j(\bv)$ if $j' = j$.

When $n = O(kn_0/\alpha + k^2\alpha^{-2}\beta^{-2}\log(1/\delta))$ is sufficiently large, with probability at least $1-\delta/4$, we have $|I_{j^*}| \ge n_0$. Conditioned on $I_{j^*}$, the data points $((\bv_i, \z_i\sps j))_{i\in I_{j^*}}$ are distributed independently from $\mD_{j^*}$, and thus by the guarantee of the agnostic learner, with probability at least $1-\delta/2$,
\[
\E_{(\bv,z)\sim \mU_{j^*}}[zh\sps {j^*}(\bv)] \ge \beta,
\]
which implies
\[
\E_{(\bv,\y)\sim \mD}[\langle \y - \bv, w_{j^*}(\bv)\rangle] = 2\Pr[\ell_\z = j^*]\E_{(\bv,z)\sim \mU_{j^*}}[zh\sps {j^*}(\bv)] \ge \alpha\beta/3k.
\]

We have thus shown that with probability at least $1-3\delta/4$, there exists $j$ such that
\[
\E_{(\bv,\y)\sim \mD}[(\y - \bv)w_{j}(\bv)] \ge \alpha\beta/3k.
\]
By estimating the values of $\E_{(\bv,\y)\sim \mD}[\langle \y - \bv, w_{j}(\bv)\rangle]$ using $O(\alpha^{-2}\beta^{-2}k^2\log (k/\delta))$ fresh examples, we can make sure that with probability at least $1-\delta$, we output a $\tilde w$ among the $w_j$'s such that
\[
\E_{(\bv,\y)\sim \mD}[\langle \y - \bv, \tilde w(\bv)\rangle] \ge \alpha\beta/6k. \qedhere
\]
\end{proof}

%% file: hardness.tex
\subsection{Agnostic Learning from Auditing}
\label{sec:red-hardness}
Now we prove the reverse direction of the reduction by constructing an agnostic learner for a class $\mH$ using an auditor (\Cref{thm:product-hardness}). For the most general statement, instead of considering the auditing task for $\mH^k$ as in \Cref{thm:red}, we need to consider a slightly different class $\tilde \mH^k$. But as long as $\mH$ is closed under coordinate-wise affine transformations of the inputs, we can choose $\tilde \mH$ to be the same as $\mH$. In particular, when $\mH$ is the class of halfspaces, by our reduction, classic hardness results on agnostically learning halfspaces implies hardness of auditing for halfspaces (\Cref{thm:halfspace-auditing}).

For a vector $\bv\in \Delta_k$ with $k\ge 2$, define $\lift(\bv)\in \Delta_k$ by
\begin{equation}
\label{eq:lift}
\lift(\bv):= \frac 13 \bv + \frac 13 \e_{1} + \frac 13 \e_{2}.
\end{equation}
\begin{theorem}
\label{thm:product-hardness}
For $k\ge 2$,
let $\mH$ be a family of functions $h:\Delta_k\to [-1,1]$ closed under negation. Let $\tilde \mH$ be a family of functions $\tilde h:\Delta_k\to [-1,1]$ such that for every $h\in \mH$, there exists $\tilde h\in \tilde \mH$ satisfying $\tilde h(\lift(\bv)) = h(\bv)$ for every $\bv \in \Delta_k$. Given any $(2\alpha/3,2\beta/3)$ auditor for $\tilde \mH^k$, we can construct an $(\alpha,\beta)$ weak agnostic learner for $\mH$ with the same sample complexity, time complexity, and failure probability parameter.
\end{theorem}

We will in fact derive this result from a more general statement where the class of auditors is not necessarily a product set.

\begin{theorem}
\label{thm:hardness}
Let $\mH$ be a family of functions $h:\Delta_k\to [-1,1]$ and let $\mW$ be a family of functions $w:\Delta_{k} \to [-1,1]^{k}$. Let $\lambda$ be a positive real number. Assume that for every $h\in \mH$ there exists $w\in \mW$ such that
\begin{equation}
\label{eq:extended-class}
w(\lift(\x))_{1} - w(\lift(\x))_{2} = \lambda h(\x) \quad \text{for every $\x\in \Delta_k$}.
\end{equation}
Given any $(\lambda\alpha/3,2\beta/3)$ auditor for $\tilde \mW$, we can construct an $(\alpha,\beta)$ weak agnostic learner for $\mH$ with the same sample complexity, time complexity, and failure probability parameter.
\end{theorem}
\begin{proof}
We construct a weak agnostic learner for $\mH$ using an auditor for $\mW$. Let $(\x_1,z_1),\ldots,(\x_n, z_n)$ be the input data points in the weak agnostic learning task drawn i.i.d.\ from a distribution $\mU$. For every input data point $(\x_i,z_i)\in\Delta_k\times[-1,1]$, the learner generates a corresponding data point $(\bv_i,\y_i)\in\Delta_{k}\times \{\e_1,\ldots,\e_{k}\}$ for the auditing task by setting 
\begin{align*} 
\bv_i &= \lift(\x_i), \\
\y_i \sim \bdv^*_i &\defeq  \bv_i + \frac 13 z_i(\e_{1} - \e_{2}).
\end{align*}
Note that $\bdv^*_i \in \Delta_{k}$, and thus it can be interpreted as a distribution over $\{\e_1,\ldots,\e_{k}\}$. Let $\mD$ denote the distribution of $(\bv_i,\y_i)$. The intuition is that since $\bv^*$ favors either $\e_1$ or $\e_2$ over $\bv$ depending on $\z$, telling the difference between $\bv_i$ and $\bdv^*_i$ for an auditor requires  learning $\z$.

Formally, by our assumption, for any $h\in \mH$, there exists $w\in \mW$ satisfying \eqref{eq:extended-class}, and thus
\[
\E_\mD[\langle \y - \bv, w(\bv)\rangle] = \frac 13\E_\mU[ z\langle \e_{1}- \e_{2}, w(\lift(\x))\rangle] = \frac \lambda 3\E_\mU[zh(\x)].
\]
Therefore, if $\E_\mU[zh(\x)] \ge \alpha$ for some $h\in \mH$, then $\E_\mD[\langle \y - \bv, w(\bv)\rangle]\ge \lambda\alpha/3$ for some $w\in \mW$, and with high probability, the auditing algorithm will produce some $w':\Delta_{k}\to [-1,1]^{k}$ such that $\E_\mD[\langle \y - \bv, w'(\bv)\rangle]\ge 2\beta/3$.
Defining $h':\Delta_k\to [-1,1]$ such that
\[
h'(\x) = \frac 12(w(\lift(\x))|_{1} - w(\lift(\x))|_{2}) \quad \text{for every }\x\in \Delta_k,
\]
we have
\begin{align*}
\E_\mD[\langle \y - \bv, w'(\bv)\rangle] &= \frac 13\E_\mU[ z\langle \e_{1}- \e_{2}, w'(\lift(\x))\rangle]\\ 
&= \frac 13\E_\mU[ z\langle  w'(\lift(\x))|_1 - w'(\lift(\x))|_2 \rangle]\\
&= \frac 2 3\E_\mU[zh'(\x)].
\end{align*}
Therefore, $\E_\mD[\langle \y - \bv, w'(\bv)\rangle]\ge 2\beta/3$ implies $\E_\mU[zh'(\x)]\ge \beta$. We have thus constructed an $(\alpha,\beta)$-weakly agnostic learning algorithm which returns $h'$ as output.
\end{proof}

We now complete the proof of Theorem \ref{thm:product-hardness}.
\begin{proof}[Proof of Theorem \ref{thm:product-hardness}]
    By our assumption about $\tilde \mH$, for every $h \in \mH$ there exist $\tilde h_1, \tilde h_2$ such that 
    \[ \tilde{h}_1(\lift(\bv)) = h(\bv), \ \tilde{h}_2(\lift(\bv)) = -h(\bv).\]
    We consider any $\tilde h \in (\tilde \mH)^k$ whose first two co-ordinates are $\tilde{h}_1$ and $\tilde{h}_2$, so that their difference is $2h(\bv)$. We now apply Theorem \ref{thm:hardness} with $\mW = (\tilde \mH)^k$ and $\lambda =2$.
\end{proof}

\eat{
We defer the proofs of \Cref{thm:product-hardness,thm:hardness} to \Cref{app:red}.
}

\section{Hardness of Auditing for Decision Calibration}
\label{sec:hardness}
Our tight connection between auditing and learning established in the previous section allows us to transfer hardness results from learning to auditing.
We apply this machinery to show hardness of auditing for specific function classes.
Under standard complexity-theoretic assumptions, we show that auditing for decision calibration (\Cref{def:dec}) cannot be solved in time $\poly(k)$.
\begin{theorem}[Hardness of Decision Calibration]
\label{thm:hardness-dec}
For $k\in \Z_{> 0}$, let $W_k$ be the class $W$ used in the definition of decision calibration (\Cref{def:dec}).
Under standard hardness assumption on refuting random $t$-XOR (\Cref{assumption-1} below), for any $C > 2$ and any sufficiently large $k$, there is no $(1/3 - 1/C, 1/k^C)$-auditing algorithm for $W_k$ that runs in time $O(k^C)$ and achieves success probability at least $3/4$.
\end{theorem}
We also prove a related result showing hardness of auditing for the product class of halfspaces. Let $\mH_\hs$ be the class of half-spaces over $\Delta_k$. That is, $\mH_\hs$ consists of all functions $h:\Delta_k\to [-1,1]$ that can be written as $h(\bv) = \sign(\ba\cdot \bv + b)$ for some $\ba\in \R^k$ and $b\in \R$. We prove the following theorem showing that $\mH_\hs^k$ does not allow $\poly(k)$-time auditing:
\begin{theorem}[Hardness of Halfspace Calibration]
\label{thm:halfspace-auditing}
Under standard hardness assumption on refuting random $t$-XOR (\Cref{assumption-1} below),
for any $C > 2$, there is no algorithm that, for every sufficiently large $k\in \Z_{> 0}$, solves $(2/3 - 1/C,1/k^C)$-auditing for $\mH_\hs^k$ in time $O(k^C)$ and achieves success probability at least $3/4$. 
\end{theorem}

We combine reductions from \Cref{sec:red-hardness} with existing hardness results of agnostically learning halfspaces to prove the two theorems above.
There are many results showing hardness of agnostic learning for halfspaces under various assumptions, for instance see \cite{FeldmanGKP06, GuruswamiR06}. The strongest results for improper learning are due to Daniely based on the hardness of refuting random $t$-XOR-Sat \cite{amit-daniely}. 

\begin{assumption}[Random $t$-XOR Assumption \cite{amit-daniely}]
\label{assumption-1}
There exist constants $\eta \in (0,1/2)$ and $c > 0$ such that for any $t\in \Z_{> 0}$, there is no $\poly(m)$-time algorithm $A$ that satisfies the following properties for any sufficiently large $n\in \Z_{> 0}$ and $m = \lfloor n^{c\sqrt t\log t}\rfloor$:
\begin{itemize}
\item given any size-$m$ collection of $t$-XOR clauses on $n$ variables where at least $1 - \eta$ fraction of the clauses are satisfiable, algorithm $A$ outputs \acc\ with probability at least $3/4$;
\item with probability at least $q(n) = 1 - o(1)$ over a uniformly randomly chosen size-$m$ collection of $t$-XOR clauses on $n$ variables, given the collection as input, algorithm $A$ outputs \rej\ with probability at least $3/4$.
\end{itemize}
\end{assumption}

\begin{theorem}[\citep{amit-daniely}]
\label{thm:halfspace-learning}
Under \Cref{assumption-1},
for any $C > 2$, there is no algorithm that, for every sufficiently large $k\in \Z_{> 0}$, solves $(1 - 1/C,1/k^C)$-agnostic learning for $\mH_\hs$ with success probability at least $3/4$ and runs in time $O(k^C)$.  
\end{theorem}
The original result by \citet{amit-daniely} was stated for the Boolean cube instead of $\Delta_k$, but the result extends to $\Delta_k$ by taking an affine injection from the Boolean cube $\{-1,1\}^{k-1}$ to $\Delta_k$.

We prove \Cref{thm:hardness-dec} and \Cref{thm:halfspace-auditing} by combining \Cref{thm:halfspace-learning} with \Cref{thm:product-hardness}  and \Cref{thm:hardness} from \Cref{sec:red-hardness}. The following simple claim is convenient for our proof and it follows immediately from the definition of $\lift(\cdot)$ in \eqref{eq:lift}.
\begin{claim}
\label{claim:closed}
For $\ba\in \R^k,b\in \R$, define $\ba' = 3\ba,  \ b' = b - \ba\sps 1 - \ba\sps 2$. Then for every $\bv\in\Delta_k$,
\[
\langle \ba',\lift(\bv)\rangle + b' = \langle \ba,\bv\rangle + b.
\]
\end{claim}

\begin{proof}[Proof of \Cref{thm:hardness-dec}]
Any function $h\in \mH_\hs$ can be expressed as $h(\bv) = \sign(\langle \ba,\bv - b\rangle)$ for $\ba\in \R^k$ and $b\in \R$.
Define $\ba' = 3\ba,  \ b' = b - \ba\sps 1 - \ba\sps 2, \bg = \e_1$ and $\bg = -\e_1$. The function $w$ mapping $\bv'$ to $\one(\langle \ba', \bv' \rangle > b')\bg + \one(\langle \ba', \bv' \rangle \le b')\bg'$ belongs to $W_k$, and 
\[
w(\lift(\bv))\sps 1 - w(\lift(\bv))\sps 2 = \one(\langle \ba', \lift(\bv) \rangle > b') -  \one(\langle \ba', \lift(\bv) \rangle \le b')= \sign(\langle \ba', \lift(\bv) - b'\rangle) = h(\bv).
\]
Therefore, by \Cref{thm:hardness}, any $(1/3 - 1/C, 1/k^C)$-auditing algorithm for $W_k$ implies a $(1 - 3/C, 3/2k^C)$-agnostic learning algorithm for $\mH_\hs$ with the same sample complexity, running time, and failure probability. The proof is completed by \Cref{thm:halfspace-learning}.
\end{proof}

\begin{proof}[Proof of \Cref{thm:halfspace-auditing}]
By \Cref{claim:closed}, $\mH_\hs$ is closed under the $\lift$ operation, namely for every $h \in \mH_\hs$ we can construct $h'\in \mH_\hs$ which satisfies $h'(\lift(\bv)) = h(\bv)$ for every $\bv\in \Delta_k$. 
Assume for the sake of contradiction that an auditing algorithm for $\mH_\hs^k$ as described in the theorem exists. By \Cref{thm:product-hardness}, such an algorithm implies a $(1 - 3C/2, 3/2k^C)$-weak agnostic learner for $\mH_\hs$ that runs in time $O(k^C)$ for any sufficiently large $k$, contradicting \Cref{thm:halfspace-learning}.
\end{proof}

%% file: upper.tex
\section{Kernel Algorithms for Auditing Calibration}
\label{sec:kernel}

In this section, we give efficient auditing algorithms for weighted calibration where the weight family $\mW$ consists of functions from a reproducing kernel Hilbert space (RKHS). In \Cref{sec:multinomial}, we discuss a special case using the multinomial kernel, which is important for our efficient auditors for projected smooth calibration in \Cref{sec:smooth}.

It is well known that learning for functions with bounded norm in an RKHS with convex losses is feasible by solving a convex program. Here we observe that the simple structure of the correlation objective $\E[zw(\bv)]$ in agnostic learning makes it possible to optimize, even without solving a convex program,  just using $O(n^2)$ kernel evaluations, via Algorithm \ref{alg:kernel-weak}. The algorithm and its analysis are not novel and are similar in nature to the kernel ridge regression algorithm (see e.g.\ \cite{kernel-ridge}). 
Based on our connection between auditing and learning shown in \Cref{sec:audit}, we  give a similar kernel evaluation based  algorithm for multi-class auditing, which we present in \Cref{alg:kernel-audit}.

Let $\mD$ be a distribution over $\Delta_k\times [-1,1]$.
Consider a positive definite kernel $\kn:\Delta_k\times \Delta_k \to \R$ and the corresponding RKHS $\Gamma$ consisting of functions $w:\Delta_k\to \R$. We assume that the kernel can be evaluated efficiently. Let $\varphi_\bv\in \Gamma$ denote the function $\kn(\bv,\cdot)$. By the reproducing property, 
\[
w(\bv) = \langle w, \varphi_\bv \rangle_\Gamma \quad\text{for every $w\in \Gamma$ and $\bv\in \Delta_k$.}
\]
Define $B_\Gamma(r)$ to be the set of $w\in \Gamma$ satisfying $\|w\|_\Gamma^2:= \langle w,w\rangle_\Gamma \le r^2$. 
For $s > 0$, assume that $\kn(\bv,\bv) \le s^2$ for every $\bv\in \Delta_k$. That is, $\|\varphi_\bv\|_\Gamma \le s$.
Under this assumption, for any $w\in B_\Gamma(1/s)$ and $\bv \in \Delta_k$, we have 
\[ |w(\bv)| = |\langle w, \varphi_\bv\rangle_\Gamma| \le \|w\|_\Gamma \|\varphi_\bv\|_\Gamma \le (1/s)\cdot s \le 1.\]

\setlength{\algomargin}{3ex}

\begin{algorithm}[t]
\caption{Kernel Algorithm for Weak Agnostic Learning}\label{alg:kernel-weak}
\SetKwInOut{Input}{Input}\SetKwInOut{Output}{Output}
\Input{Data points $(\bv_1, z_1),\ldots, (\bv_n, z_n)\in \Delta_k\times [-1,1]$.}
\Output{Function $w_2:\Delta_k\to [-1,1]$.}
\Begin{
$\lambda \leftarrow \left(\sum_{i=1}^n\sum_{j=1}^nz_iz_j\ker(\bv_i, \bv_j)\right)^{1/2}$\;
$w_2(\bv) \leftarrow \fr{\lambda s}\sum_{i=1}^nz_i\ker(\bv_i, \bv)$ for every $\bv\in \Delta_k$ (in the degenerate case where $\lambda = 0$, set $w_2(\bv) \gets 0$)\;
\Return $w_2$\;
}
\end{algorithm}

The following theorems are proved in \Cref{sec:proof-kernel}.

\begin{theorem}\label{thm:kernel-weak}
    For $n =  O(r^2s^2\alpha^{-2}\log(1/\delta))$, Algorithm \ref{alg:kernel-weak} is an $(\alpha, \alpha/3rs)$ agnostic learner for the class $B_\Gamma(r)$ with failure probability at most $\delta$. Moreover, it always returns a function from $B_\Gamma(1/s)$.
\end{theorem}

\begin{theorem}\label{thm:kernel-audit}
    For $n =  O(kr^2s^2\alpha^{-2}\log(1/\delta))$, Algorithm \ref{alg:kernel-audit} is an $(\alpha, \alpha/3rs)$ auditor for the class $B_\Gamma(r)^k$ with failure probability at most $\delta$. Moreover, it always returns a function from $B_\Gamma(1/s)^k$.
\end{theorem}
\begin{algorithm}[t]
\caption{Kernel Algorithm for Auditing}\label{alg:kernel-audit}
\SetKwInOut{Input}{Input}\SetKwInOut{Output}{Output}
\Input{Data points $(\bv_1, \y_1),\ldots, (\bv_n, \y_n)\in \Delta_k\times \mE_k$.}
\Output{Function $w_2:\Delta_k\to [-1,1]^k$.}
\Begin{
$\z_i \gets \y_i - \bv_i$ for every $i = 1,\ldots,n$\;
For $i = 1,\ldots,n$ and $\ell = 1,\ldots,k$, let $\z_i\sps \ell$ denote the $\ell$-th coordinate of $z_i$\;
$\lambda\sps \ell \leftarrow \left(\sum_{i=1}^n\sum_{j=1}^n\z_i\sps \ell\z_j\sps \ell\ker(\bv_i, \bv_j)\right)^{1/2}$ for every $\ell = 1,\ldots,k$\;
$w_2\sps \ell(\bv) \leftarrow \fr{\lambda\sps \ell s}\sum_{i=1}^n\z_i\sps \ell\ker(\bv_i, \bv)$ for every $\ell = 1,\ldots,k$ and $\bv\in \Delta_k$ (in the degenerate case where $\lambda\sps \ell = 0$, set $w_2\sps\ell(\bv) \gets 0$)\;%
\Return $w_2$ such that $w_2(\bv) = \big(w_2\sps 1(\bv),\ldots,w_2\sps k(\bv)\big)$ for every $\bv \in \Delta_k$\;
}

\end{algorithm}

\subsection{Auditing for the Multinomial Kernel}
\label{sec:multinomial}

A kernel that will be of particular importance for us is the multinomial kernel. We follow the elegant formulation from \cite{GoelKKT17}. 

\begin{definition}\cite{GoelKKT17}
\label{def:mm-kernel}    
For any vector $\bv= (v_1,\ldots,v_k)\in \Delta_k$ and  tuple $t = (t_1,\ldots,t_d)\in [k]^d$, define $\bv^t$ to be the product $v_{t_1}\cdots v_{t_d}$. Define $\psi:\Delta_k\to \R^{1 + k + \cdots + k^d}$ such that $\psi(\bv)$ is a vector whose coordinate indexed by $t\in T_d:=[k]^0\cup [k]^1\cup \cdots \cup [k]^d$ is $\bv^t$. The degree $d$ multinomial kernel is given by 
\[ \kn_d(\bv,\bv') = \sum_{i=0}^d(\bv\cdot \bv')^i = \psi(\bv)\cdot \psi(\bv'). \] We denote its RKHS as $\Gamma(d)$. 
\end{definition}

Instantiating \Cref{thm:kernel-audit} for the degree $d$ multinomial kernel gives the following result.

\begin{lemma}\label{lem:poly-weak}
For all, $r \geq 0$ and $d \geq 0$, Algorithm \ref{alg:kernel-audit} with $n =  O(kr^2d\log(1/\delta)/\alpha^2)$ samples is an $(\alpha, \alpha /(3r\sqrt{d}))$-auditor for the class $(B_{\Gamma(d)}(r))^k$ with failure probability at most $\delta$. Moreover, it always returns a function from $(B_{\Gamma(d)}(1/\sqrt{d}))^k$ in time $\poly(n,k,d)$.
\end{lemma}
\begin{proof}
Observe that for $\bv \in \Dk$,
\[ \|\varphi_\bv\|^2_{\Gamma(d)} = \kn_d(\bv, \bv) =  \sum_{i=0}^d(\bv\cdot \bv)^i \leq d\]
since $\norm{\bv}_2^2 \leq 1$ for $\bv \in \Dk$. We can thus apply Theorem \ref{thm:kernel-audit} with $s = \sqrt{d}$ to get the claimed bound.
\end{proof}

This gives a faster auditor for the notion of low-degree calibration defined by \cite{GopalanKSZ22}.

\begin{definition}\cite{GopalanKSZ22}
Let $P(d,1)$ denote the set of multivariate degree $d$ polynomials 
\begin{align*} 
p(v_1,\ldots, v_k) &= \sum_{e: \deg(e) \leq d} w_e\prod_{i}v_i^{e_i}\\
\text{where} \ \forall \bv \in \Dk, \ |p(\bv)| &\leq 1 ,\\
\sum_{e: \deg(e) \leq d} |w_e| &\leq 1.
\end{align*}
We say that a predictor $p$ is $\alpha$ degree-$d$ calibrated if $\CE_{P(d,1)^k}(\mD) \leq \alpha$.
\end{definition}

\cite{GopalanKSZ22} give an $(\alpha, \alpha/k^d)$-auditor for $P(d,1)^k$ which runs in time $O(k^d)$ by enumerating over all $k^d$ monomials. Algorithm \ref{alg:kernel-audit} implies a better auditor which is polynomial in both $k$ and $d$.

\begin{corollary}
\label{cor:low-deg}
There is an $(\alpha, \alpha/3\sqrt{d})$-auditor for $P(d,1)^k$ that with success probability at least $1-\delta$, sample complexity $n = O(kd\log(1/\delta)/\alpha^2)$ and time complexity $\poly(n,k,d)$.
\end{corollary}
\begin{proof}
Any polynomial $p \in P(d,1)$ can be written as $p(\bv) = \sum_{t\in T_d}w_t \bv^t$ for every $\bv\in \R^k$, where $w_t\in \R$ for every $t\in T_d$ and $\sum_{t\in T_d}|w_t|\le 1$. 
We define a vector $\psi^p\in \R^{1 + k + \cdots + k^d}$ whose coordinate indexed by $t\in T_d$ is $w_t$.
It follows that
    \begin{align*} 
         & p(\bv) = \sum_{t\in T_d}w_t \bv^t = \psi^p \cdot \psi(\bv),\\
        & \norm{p}_{\Gamma(d)}^2 \le \|\psi^p\|_2^2 = \sum_{t} w_t^2  \leq \lt(\sum_t |w_t|\rt)^2 \leq 1.
    \end{align*}
    Hence $P(d,1) \subseteq \mB_{\Gamma(d)}(1)$. Hence the claimed bound follows from Lemma \ref{lem:poly-weak} with $r =1$.
\end{proof}

%% file: smooth.tex
\section{Efficient Auditing for Projected Smooth Calibration}
\label{sec:smooth}

In this section, we prove the following theorem showing an efficient kernel-based auditing algorithm for projected smooth calibration (\Cref{def:p-smooth}).

\begin{theorem}
\label{thm:p-smooth}
There exists $c >0$ so that for any $\alpha,\delta\in (0,1/2)$ and $m\in [2,k]$, there is an $(\alpha, 1/m^{O(1/\alpha)})$ auditor for $m$-projected smooth calibration (and hence also for $m$-subset smooth calibration), with success probability at least $1-\delta$, sample complexity $n = O(km^{O(1/\alpha)}\log(1/\delta))$, and running time $\poly(n,k,1/\alpha)$.
\end{theorem}\PG{Can we check if the $k^2$ dependency is correct}
Even when we consider subset calibration over arbitrary subsets, which corresponds to taking $m =k$, the running time of the auditor is $k^{O(1/\alpha)}$, which is polynomial in $k$ for every fixed $\alpha$. In the next section (\Cref{sec:psmooth-lower}), we show that this  running time cannot be improved to $\poly(k,1/\alpha)$ under standard complexity-theoretic assumptions.
At the end of this section, we show that the dependence on $\alpha$ can be improved if we consider sigmoid functions instead of all $1$-Lipschitz functions.

\eat{
\begin{definition}

\label{def:psc}
Let $T$ be the class of all $1$-Lipschitz functions $\tau:\bR\to [-1,1]$. Let $\mu$ be a distribution over $X\times\{\e_1,\ldots,\e_k\}$. We define the projected smooth calibration errors of $f:X\rightarrow \Delta_k$ as follows:
\begin{align*}
\psc(f) & :=\sup_{\tau\in T}\sup_{w\in [-1,1]^k}|\bE_{(x,y)\sim\mu}[\langle w, y - f(x)\rangle\tau(\langle w,f(x)\rangle)]|,\\
\psc_1(f) & :=\sup_{\tau\in T}\sup_{w\in [-1,1]^k}\|\bE_{(x,y)\sim\mu}[(y - f(x))\tau(\langle w,f(x)\rangle)]\|_1.
\end{align*}
\end{definition}
\begin{claim}
$\psc(f) \le \psc_1(f)$.
\end{claim}}

We prove Theorem \ref{thm:p-smooth} using Algorithm \ref{alg:kernel-audit}, together with polynomial approximations. Low degree polynomial approximations have been used successfully for agnostic learning, starting with the work of \cite{KKMS08}. The important work of \cite{SSS11} showed that one can improve the efficiency of such learning algorithms by kernelizing them.

\eat{By combining polynomial approximation results with the auditing algorithm for the multinomial kernel, we get an efficient auditor for projected smooth calibration and subset smooth calibration. The following

\begin{theorem}
\label{thm:p-smooth}
For any $\alpha > 0$ and $m \ge 2$, \Cref{alg:kernel-audit} with the degree-$d$ multinomial kernel for $d = O(1/\alpha)$ is an $(\alpha, \alpha/m^{O(1/\alpha)}))$ auditor for $m$-projected smooth calibration, and hence also for $m$-subset smooth calibration. The sample complexity and running time are both $k^{O(1)} m^{O(1/\alpha)}$.
\end{theorem}

 This result is in stark contrast to our hardness result for auditing halfspaces \Cref{thm:halfspace-auditing}, which rules out a similar running time for that problem.

The rest of this section is devoted to proving }

Using results from \cite{GoelKKT17} and \cite{Sherstov}, we will show the following bound on multivariate polynomials obtained by composing bounded univariate polynomials with innner products.

\begin{lemma}\label{lem:kernel-bound}
 Let $p$ be a univariate polynomial of degree $d$ so that $|p(u)| \leq 1$ for $u \in [-1,1]$. Let $p_\ba(\bv) = p(\ba\cdot \bv)$ where $\ba\in [-1,1]^d$ and $\bv \in \Dk$. Then $p_\ba \in \Gamma(d)$ and 
\[ \|p_\ba\|_{\Gamma(d)}^2 \le \max(4, 4\|\ba\|_2)^{2d}. \]
\end{lemma}
We prove \Cref{lem:kernel-bound} using the following two lemmas from the literature:
\begin{lemma}\citep[Lemma 2.7]{GoelKKT17}\label{lem:gkkt}
Let $p = \sum_{i=0}^d \eta_iu^i$ be a univariate polynomial of degree $d$ and $p_\ba(\bv) = p(\ba\cdot \bv)$ for $\ba \in [-1,1]^k$ and $\bv \in \Dk$. Then 
\[ \|p_\ba\|_{\Gamma^d}^2 \leq \sum_{i=0}^d\eta_i^2\|\ba\|_2^{2i} \leq \max(1, \|\ba\|_2)^{2d}\sum_{i=0}^d \eta_i^2 . \]
\end{lemma}

\begin{lemma}\cite[Lemma 4.1]{Sherstov}\label{lem:sherstov}
    For a degree $d$ polynomial $p(u) = \sum_{i=0}^d  \eta_i u^i$ satisfying $|p(u)| \leq 1$ for $u \in [-1,1]$, it holds that $\sum_{i=0}^d|\eta_i| \leq 4^d$.  
\end{lemma}

\begin{proof}[Proof of Lemma \ref{lem:kernel-bound}]
By Lemma \ref{lem:sherstov}, we can  bound 
\[ \sum_{i=0}^d|\eta_i|^2 \leq \left(\sum_{i=0}^d|\eta_i|\right)^2 \leq 4^{2d}.\] 
We plug this bound into Lemma \ref{lem:gkkt} to get
\[ \|p_\ba\|_{\Gamma^d}^2 \le \max(4, 4\|\ba\|_2)^{2d}.\qedhere \]
\end{proof}

Let $\Lip$ denote the set of all bounded $1$-Lipschitz functions $\phi:[0,1] \to [-1,1]$. 
\eat{Next we derive a similar bound for all functions in $\Lip$, the difference being that we need degree $1/\eps$ rather than $\log(1/\eps)$.  }
We use an approximation result for arbitrary Lipschitz functions using Jackson's theorem \cite{Cheney}, together with a rescaling argument to ensure boundedness. A similar argument for the ReLU function appears in \cite[Lemma 2.12]{GoelKKT17}.

\begin{lemma}\label{lem:lip-approx}
    There exists a constant $c' > 0$ such that for any $\phi \in\Lip$ and any $\eps > 0$, there exists a univariate polynomial $p(t)$ with $\deg(p) \leq c'/\eps$ such that for $t \in [-1,1]$,
    \begin{itemize}
    \item $\lt|\phi(t) - p(t) \rt| \leq \eps$ .
    \item $p(t) \in [-1,1]$.
    \end{itemize}
\end{lemma}
\begin{proof}%
By Jackson's theorem \cite{Cheney}, there exist a polynomial $p(t)$ so that
$\lt|\phi(t) - p(t) \rt| \leq \eps/2$ for $t \in [-1,1]$ where $\deg(p) \leq  O(1/\epsilon)$. 
Since $|\phi(t)| \leq 1$, $|p(t)| \leq 1 + \eps/2$. Now let $p_\phi(t) = p(t)/(1+ \eps/2)$ so that $|p_\phi(t)| \leq 1$. We then bound
\begin{align*}
    |p_\phi(t)  - \phi(t)| &= \fr{1 + \eps/2}|(p(t) - (1 + \eps/2)\phi(t))| \\
    &\leq  \fr{1+ \eps/2}(|p(t) - \phi(t)| + \eps/2|\phi(t)|)\\
    &\leq \frac{\eps}{1+ \eps/2} \leq \eps.
    \qedhere
\end{align*}
\end{proof}

\eat{
\begin{lemma}\label{lem:lip-kernel}
    For $w \in [-1,1]^k$ let $p_{s,w}(v) = p_s(w\cdot v)$. Then 
    \begin{enumerate}
        \item $p_{\sigma, w} \in \Gamma(d)$ for $d = c'L/\eps$. 
        \item $|p_{\sigma,w}(v) - \sigma(w\cdot v)| \leq \eps$ for all $v \in \Dk$.
        \item $\norm{p_{w, \sigma}}_{\Gamma(d)} \leq (c'_1\|w\|)^{c'_2L/\eps)}$ for some constants $c'_1, c'_2$. 
    \end{enumerate}   
\end{lemma}
The proof is essentially identical to Lemma \ref{lem:sig-kernel}. 
}

Combining Lemmas \ref{lem:lip-approx} and \ref{lem:kernel-bound}, we have the following corollary.
\begin{corollary}
\label{cor:kernel}
    For any $\phi \in \Lip$, and $\eps >0$, let  $p$ be as in  Lemma \ref{lem:lip-approx}. For $\ba \in [-1,1]^k$, let $p_{\ba}(\bv) = p(\ba\cdot \bv)$. Then $p_{\ba} \in \Gamma(d)$ for $d = O(1/\eps)$ and 
    \begin{itemize}
        \item $\abs{p_{\ba}(\bv) - \phi(\ba\cdot \bv)} \leq \eps$, \ for every $v \in \Dk$.
        \item $\norm{p_{\ba}}_{\Gamma(d)} \leq c_1\max(1, \norm{\ba}_2)^{c_2/\eps}$.
    \end{itemize}
\end{corollary}
\eat{
\begin{proof}
    We let $p_{\phi,w}(v) = p_\phi(w\cdot v)$. The uniform approximation property follows since $w\cdot \Dk \in [-1,1]$ for any $w \in [-1,1]^k$, and $|p_s(u) - s(u)| \leq \eps$ for $u \in [-1,1]$. We bound the RKHS norm of $p_{s,w}$ using Lemma \ref{lem:kernel-bound}.  
\end{proof}
}
We now complete the proof of \Cref{thm:p-smooth}.

\begin{proof}[Proof of \Cref{thm:p-smooth}]
We claim that if $\psCE_m(\mD) \ge \alpha$, then $\CE_{(\mB_{\Gamma(d)}(r))^k}(\mD) \geq \alpha/2$ for some $r = m^{O(1/\alpha)}$. To see this, take $\psi \in \pLip^k$ so that
\[\E_\mD[\ip{\y^* - \bv}{\psi(\bv)} \geq \alpha.\]
For every $i = 1,\ldots,k$, there exists $\ba_i\in [-1,1]^k$  and $\phi\in \Lip$ such that $\psi\sps i(\bv) = \phi\sps i(\ip{\ba_i}{\bv})$ for $i \in [k]$ where $\norm{\ba_i}_2 \leq \sqrt m$.
By Corollary \ref{cor:kernel},  there exists $p\sps i \in \Gamma(d)$ where $d = O(1/\alpha)$ such that 
\begin{align*} 
& \norm{\psi_i(\bv) - p\sps i(\bv)}_\infty \leq \alpha/4, \\
& \norm{p\sps i}_{\Gamma(d)} \leq (c_1\max(1, \norm{\ba_i}_2))^{c_2/\alpha} \leq (c_1\sqrt{m})^{c_2/\alpha}.
\end{align*}
Hence $p\sps i \in \mB_{\Gamma(d)}(r)$ for each $i$ for $r = m^{O(1/\alpha)}$. 

Define $p(\bv) = (p\sps 1(\bv),\ldots,p\sps k(\bv))\in \R^k$. By the triangle inequality
\begin{align*}
\abs{\E_\mD[\ip{\y^* - \bv}{p(\bv)} - \E_\mD[\ip{\y^* - \bv}{\psi(\bv)}]} &= \abs{\E_\mD[\ip{\y^* - \bv}{p(\bv) - \psi(\bv)}]}\\
&\leq \E_\mD\lt[\lt|\ip{\y^* - \bv}{p(\bv) - \psi(\bv)}\rt|\rt]\\
& \leq \E_\mD\lt[\norm{\y^* - \bv}_1\norm{p(\bv) - \psi(\bv)}_\infty\rt]\\
& \leq 2\cdot \frac{\alpha}{4} \leq \frac{\alpha}{2}
\end{align*}
where we use $\norm{\y^* - \bv}_1 \leq 2$.
As a result we have
\[ \E_\mD[\ip{\y^* - \bv}{p(\bv)}] \geq \E_\mD[\ip{\y^* - \bv}{\psi(\bv)}] - \alpha/2 \geq \alpha - \alpha/2 = \alpha/2. \]

We now apply Lemma \ref{lem:poly-weak} with the weight functions $(\mB_{\Gamma(d)}(r))^k$ where $d = O(1/\alpha), r = m^{O(1/\alpha)}$  
to get an $(\alpha/2, \Omega(\alpha^{3/2}/m^{c/\alpha}))$-auditing algorithm. 
\end{proof}

\paragraph{Auditing for Sigmoids.} We show additionally that the exponential dependence on $1/\alpha$ in \Cref{thm:p-smooth} can be improved if we audit only for sigmoid functions. Formally we use the $\tanh$ function rather then the sigmoid, since we want the range to be $[-1,1]$ in order to approximate the sign function. Nevertheless, we refer to the family as the family of sigmoid functions.

\begin{definition}
For $L \geq 1$, define $\Sigma_L =\{g:\R \to [-1,1]\}$ to be the family of functions of the form 
    \[ g(\bv) = \tanh(L\ip{\ba}{\bv} + b)  \ \text{for}  \ \ba \in [-1,1]^k, \ b \in \R.\]
\end{definition}

In \Cref{lem:sigmoid} below we show an efficient auditor for $\Sigma_L^k$ whose running time is polynomial in $1/\alpha$ for every fixed $k$ and $L$.

Observe that $\Sigma_L$ increase monotonically with $L$, since for $L' < L$, $L'\ip{\ba}{\bv} = L\ip{\ba'}{v}$ where $\ba' = L'\ba/L \in [-1,1]^k$. The problem of agnostically learning $\Sigma_L$ over $\Dk$ is given a distribution $\mU$ on $\Dk \times \pmo$, find $g \in \Sigma_L$ that maximizes $\E_\mU[g(\bv)z]$. The problem of agnostically learning sigmoids over the unit sphere (rather than $\Dk$) was considered in the influential work of \cite{SSS11}. They work with the objective function $\min_{g \in \Sigma_L} \E|z -g(\bv)|$, but this is seen to be equivalent to $\max_{g \in \Sigma_L}[\E[g(\bv)z]$ when $z\in \{-1,1\}$. A more substantial difference is that they work in the $\ell_2$ bounded setting where $\norm{\bv}_2 \leq 1$, $\norm{\ba}_2 \leq 1$, whereas we work with $\ell_1/\ell_\infty$-bounded setting where $\norm{\bv}_1 \leq 1$ and $\norm{\ba}_\infty \leq 1$. Thus we cannot directly use their results, although our techniques are influenced by them.

 Our algorithm will use the following results about univariate approximations to the $\tanh$ function  was proved in the work of \cite{SSS11}, with subsequent proofs given by \cite{LSS14, GoelKK20}. We use the following version from \cite{GoelKK20}. 

\begin{lemma}\cite{GoelKK20}\label{lem:sig-approx}
    For $\eps \in (0,1/2), L \geq 1$ and $b\in \R$, there exists a univariate polynomial $p(t)$ with $\deg(p) \leq O(L\log(L/\eps))$ so that for $t \in [-1,1]$
    \begin{itemize} 
    \item $\lt|\tanh(Lt + b) - p(t) \rt| \leq \eps$ .
    \item $p(t) \in [-1,1]$.
    \end{itemize}
\end{lemma}

\eat{
\begin{corollary}
\label{cor:kernel}
    For $\eps > 0$, $L \geq 1$ and $w \in [-1,1]^k$,  there exists  $p_{L, w} \in \Gamma(d)$ for $d = O(L\log(L/\eps)$ so that  $\lt|\sigma_L(w\cdot v) - p_{L, w}(v) \rt| \leq \eps$ and
    $\norm{p_{L, w}}_{\Gamma(d)} \leq O(\max(1, \norm{w}_2)^d)$.
\end{corollary}
}

Following the same proof outline as Theorem \ref{thm:p-smooth} gives the following result.

\begin{theorem}\label{lem:sigmoid}
For any $\alpha\in (0,1/2), L > 1$, there is an efficient $(\alpha, \beta)$-auditor for $(\Sigma_L)^k$ calibration for 
\[ \beta = \frac{\alpha}{k^{O(L\log(L/\alpha))}} \] 
which has time and sample complexity $k^{O(L\log(L/\alpha))}$ and success probability at least $1 - 2^{-k}$.
\end{theorem}

The same techniques also yield an algorithm for agnostically learning $\Sigma_L$ under any distribution $\mU$ on $\Dk \times [-1,1]$ with similar parameters.

\eat{
There exists an absolute constant $C> 0$ such that the following holds for every $\varepsilon\in (0,1)$.
For any $f:X\to \Delta_k$, there exist a monic monomial $q:\Delta_k \to [0,1]$ of degree at most $C/\varepsilon$ and $i\in [k]$ such that
\[
|\bE[(y_i - f(x)_i)q(f(x))]| \ge \frac{\psc_1(f) - \varepsilon}{k^{C/\varepsilon}}.
\]
\end{lemma}
\begin{proof}
For every $\tau\in T$ as in \Cref{def:psc}, by Jackson's theorem, for an absolute constant $C' > 0$, there exist $\ell\le C'/\varepsilon$ and $\beta_1,\ldots,\beta_\ell\in [-2^{C'/\varepsilon}, 2^{C'/\varepsilon}]$ such that $|\tau(v) - \sum_{j=0}^\ell \beta_j v^j| \le \varepsilon/2$ for every $v\in [-1,1]$. Combining this with the definition of $\psc_1$ and the fact that $\|y - f(x)\|_1 \le 2$, we have
\[
\sup_{w\in [-1,1]^k}\|\bE[(y - f(x))\sum_{j=0}^\ell\beta_j \langle w,f(x) \rangle^j]\|_1 \ge \psc_1(f) - \varepsilon.
\]
The lemma is proved by expanding out $\langle w,f(x)\rangle^j$ as the sum of $k^j$ monomials in $f(x)$ each with coefficient bounded in $[-1,1]$ and applying the triangle inequality.
\end{proof}
}

%% file: sigmoid.tex
\section{Computational Lower Bound for Projected Smooth Calibration}
\label{sec:psmooth-lower}
The sample and time complexity of our auditing algorithm for projected smooth calibration in \Cref{sec:smooth}
 is $k^{O(1/\alpha)}$ (when setting $m = k$).
 In this section, we show that an improvement to $\poly(k,1/\alpha)$ (or just to $k^{O(\log^{0.99}(1/\alpha))}$) would violate standard complexity-theoretic assumptions:
\begin{theorem}
\label{thm:hard-psmooth}
Under a standard hardness assumption on refuting $t$-XorSat (\Cref{assumption:2}), for any $C > 0$, $\varepsilon > 0$, there is no algorithm solving $(\alpha,1/k^C)$ auditing for $k$-projected smooth calibration for every sufficiently large $k$ and every $\alpha\in (0,1/3)$ with success probability at least $3/4$ and running time $k^{O((\log(1/\alpha))^{1-\varepsilon})}$.
\end{theorem}
We use the following connection between auditing projected smooth calibration and auditing for sigmoids $\Sigma_L^k$.
\begin{lemma}
\label{lm:smooth-sigmoid}
For $\alpha,\beta\in (0,1)$ and $L > 1$,
any $(\alpha/L,\beta)$-auditing algorithm for $\Lip^k$ is an $(\alpha,\beta)$-auditing algorithm for $\Sigma_L^k$. 
\end{lemma}
\begin{proof}
For a class $\mW$ of functions $h:\Dk\to [-1,1]^k$, recall the following notion in our definition of auditing:
\[ \CE_\mW(\mD) = \sup_{w \in \mW} \lt| \E_{(\bv,\y)\sim \mD} [\ip{\y - \bv}{w(\bv)}] \rt|. \]
It suffices to show that for any distribution $\mD$ over $\Dk\times \Ek$,
\begin{equation}
\label{eq:smooth-sigmoid-1}
\CE_{\Lip^k}(\mU) \ge \frac 1 L \CE_{\Sigma_L^k}(\mU).
\end{equation}
Consider any function $g\in \Sigma_L$. By definition, there exist $\ba\in [-1,1]^k$ and $b\in \R$ such that $g(\bv) = \tanh(L \langle \ba,\bv\rangle + b)$ for every $\bv\in \Dk$. It is easy to verify that $\tanh$ is $1$-Lipschitz, and thus for any $\bv_1,\bv_2\in \Dk$,
\begin{align*}
|g(\bv_1) - g(\bv_2)| \le L |\langle \ba,\bv_1\rangle - \langle \ba,\bv_2\rangle| \le L \|\bv_1 - \bv_2\|_1.
\end{align*}
Therefore, the function $g/L$ belongs to $\Lip$, confirming \eqref{eq:smooth-sigmoid-1}.
\end{proof}
\begin{theorem}
\label{thm:hard-sigmoid}
Under a standard hardness assumption on refuting $t$-XorSat (\Cref{assumption:2}), for some fixed $\alpha > 0$, for any $C > 0$, $\varepsilon > 0$, $c\in (0,1)$, there is no algorithm that solves $(\alpha,1/k^C)$ auditor for $\Sigma_L^k$ for every sufficiently large $k\in \Z_{> 0}$ and $L:= \exp(\log^ck)$ with success probability at least $3/4$ and running time $k^{(\log L)^{1-\varepsilon}}$.
\end{theorem}
\begin{proof}[Proof of \Cref{thm:hard-psmooth}]
Let $\alpha_0$ denote the fixed constant $\alpha$ in \Cref{thm:hard-sigmoid}.
\Cref{thm:hard-psmooth} follows immediately by combining \Cref{thm:hard-sigmoid} and \Cref{lm:smooth-sigmoid}, where we choose $\alpha$ in \Cref{thm:hard-psmooth} to be $\alpha_0/L$.
\end{proof}
\subsection{Projected Smooth Calibration and Sigmoids}
Now we prove \Cref{thm:hard-sigmoid}.
Our reduction from auditing to agnostic learning lets us focus on the complexity of agnostic learning $\Sigma_L$ to understand auditing with weight functions $\Sigma_{L/3}^k$. This is formally stated below.

\begin{lemma}
\label{lm:sigmoid-reduction}
Given any $(2\alpha/3,2\beta/3)$ auditor for $\Sigma_{L/3}^k$, we can construct an $(\alpha,\beta)$ weak agnostic learner for $\Sigma_{L}$ over $\Dk$ with the same sample complexity, time complexity, and failure probability parameter.
\end{lemma}
\begin{proof}
    For $g \in \Sigma_{L/3}$, we claim there exists $g'\in \Sigma_L$ such that $g'(\lift(\bv)) = g(\bv)$. Indeed since $\lift(\bv) = \bv/3 + \e_1/3 + \e_2/3$,
    \begin{align*}
        \tanh(L\ip{w}{\bv}/3 + b) &= \tanh(L\ip{w}{3\lift(\bv)}/3 + b - Lw^{(1)}/3 - Lw^{(2)}/3)\\
        &= \tanh(L\ip{w}{\lift(\bv)} + b') \in \Sigma_L
    \end{align*}
    We now apply \Cref{thm:product-hardness} to get the stated claim.
\end{proof}

Our lower bound for agnostically learning sigmoids is obtained by tailoring Daniely's \cite{amit-daniely} reduction from refuting random $\XS$ to the $\ell_\infty/\ell_1$ bounded setting.

An instance of $t$-$\XS$ consists of $m$ clauses on $n$ variables $\{z_1, \ldots, z_n\}$ each taking values in $\pmo$. Each clause consists of exactly $t$ literals which might be variables or their negations, we assume that $x_i$ and $-x_i$ do not occur in the same clause. Thus each clause $c$ can be arithmetized as a vector in $\zo^{2n}$ of weight exactly $t$, interpreted as a subset of literals. We will let $C(t) \subset \zo^{2n}$ denote the set of valid clauses. Similarly, assignments to $z$ can be (redundantly) arithmetized as vectors in $Z \subseteq \pmo^{2n}$ where $|Z| = 2^n$. Given a clause $c_i \in C$ and $z \in Z$, $c_i\cdot z \in \{-t, \ldots, t\}$ equals the sum of literals in the clause $c_i$. An instance of $t$-$\XS$ is given by $I = \{(c_i, b_i)\}_{i = 1}^m$ where $c_i \in C(t)$ and $b_i \in \pmo$. For a clause $c \in C_t$, let $\Xor_c(z) = \prod_{i \in c}z_i$. For an instance $I$ and $z \in Z$, we define 
\[ \val(z, I) = \frac{|{i \in [m]: \Xor_{c_i}(z) = b_i}|}{m} \]
and $\val(I) = \max_{z \in Z}\val(z,I)$ to be the maximum fraction of satisfiable clauses.

A random instance of $t$-$\XS$ is one where $c_i \leftarrow C$ and $b_i \leftarrow \pmo$ are drawn uniformly and independently at random. We let $\mR$ denote the distribution on instances that this defines. An algorithm $\mA$ which maps $t$-$\XS$ instances to $\zo$  successfully refutes random $t$-$\XS$ if 
\begin{align*}
    \Pr[\mA(I) =1] & \geq  \frac{3}{4} \ \text{if} \ \val(I) \geq 1 - \eta\\
    \Pr[\mA(I) =0]] & \geq \frac{3}{4} \ \text{with probability $1 - o_n(1)$ over $I \sim \mR$}.
\end{align*}
We are interested in the asymptotics in both $t$ and $n$. 
The best known algorithms for refutation require $m = \Omega(n^{t/2})$ and it is conjectured that there are no algorithms with running time $n^{o(t)}$. 
\begin{assumption}[Random $t$-XOR Assumption \cite{amit-daniely}]
\label{assumption:2}
There exist constants $\eta \in (0,1/2)$ and $\gamma > 0$ such that for any $s > 0$, there is no $\poly(m)$-time algorithm that refutes random $t$-XorSat with $m$ clauses for any sufficiently large $n\in \Z_{> 0}$, $m = \lfloor n^{\gamma t}\rfloor$, and $t = \lfloor \log ^s(n)\rfloor$.
\end{assumption}
\begin{theorem}
\label{thm:sigmoid-hard}
Under \Cref{assumption:2},
for some fixed $\alpha > 0$, any $C > 0$, $c\in (0,1)$, and $\varepsilon >0$, there is no algorithm that solves $(\alpha,k^{-C})$-weak agnostic learning for $\Sigma_{L}$ over $\Dk$ for every sufficiently large $k\in \Z_{> 0}$ and $L:=\exp(\log^{c}k)$ with success probability at least $3/4$ and running time $k^{O(\log^{1-\varepsilon} L)}$.
\end{theorem}

\begin{proof}[Proof of \Cref{thm:hard-sigmoid}]
\Cref{thm:hard-sigmoid} follows immediately by combining \Cref{thm:sigmoid-hard} and \Cref{lm:sigmoid-reduction}.
\end{proof}

In preparation for proving \Cref{thm:sigmoid-hard}, we prove a few preliminary results. Given a set of clauses $c = \{c_i\}_{i \in [m]}$, define the function:
\[ q(c) = \max_{z \in Z}\fr{m}\sum_{i \in [m]} (c_i\cdot z)^2 \]

The following lemma is implicit in \cite{amit-daniely}
\begin{lemma}
There exists a constant $a_1$ such that
    \[\Pr_{c \leftarrow C^m}[q(c) \leq a_1 t \log(t)] \geq 1 - o_m(1) \]
\end{lemma}
where the $o_m(1)$ is exponentially small in $m$.
\begin{proof}
    Fix an assignment $z \in Z$. We view choosing $c_i \leftarrow C$ and first choosing a subset $T_i \subseteq [n]$ of variables, and then choosing their polarities $p_i \in \pmo^t$.
    
    For every $z$ and $T$, by a Chernoff bound (over the choice of $p$), there exists a constant $a_2$ so that 
    \[ Pr_{p_i \leftarrow \pmo^t}[|c_i \cdot z| \geq a_2\sqrt{t\log(1/\delta)}] \leq \delta.\]
    By a Chernoff bound over the choice of $T_i$ \cite[Theorem 4.1]{motwani-raghavan},  we have that with probability $\exp(-a_3\delta m)$ , the condition
    \[ |c_i \cdot z| \leq a_2\sqrt{t\log(1/\delta)} \]
    holds for $m(1 -2\delta)$ clauses. For such a $z$, we can bound
    \begin{align*} 
        \fr{m}\sum_{i \in [m]} ((c_i\cdot z)^2 -t) &\leq \fr{m}\sum_{i=1}^m|c_i\cdot z|^2 \\
        & \leq (1 - 2\delta)a_2t\log(1/\delta) + 2\delta t^2 \leq a_1t\log(t)
    \end{align*}
    where we choose $\delta = \log(t)/t$. 
    
    By a union bound over all $2^n$ choices of $z$, this holds for every $z$, and hence for $q(c)$ with probability $2^n\exp(-a_2\delta m)$, which is exponentially small once $m \geq a_3 nt$. 
\end{proof}

The next lemma is also proved in \cite{amit-daniely}. We use a different technique based on semi-definite programming and Grothendieck's inequality, which is more along the lines of the reduction in Fiege's work \cite{fiege-sat}.

\begin{lemma}
\label{lem:sdp}
    There is an algorithm that accepts all $c \in C^m$ such that $q(c) \geq a_1 t \log(t)$ and rejects instances such that $q(c) \leq 2a_1t\log(t)$.
\end{lemma}
\begin{proof}
We can write
\begin{align*} 
    q(c)  - t &=  \max_{z \in Z} \fr{m}\sum_{i=1}^m ((c_i \cdot z)^2 -t)\\
    &= \max_{z \in Z} \fr{m}\sum_{i \in [m]}\sum_{j \neq {j'} \in\ c_i} z_jz_{j'} 
\end{align*}
We consider the semi-definite relaxation over $v_i$ which are unit vectors in a high-dimensional space (with the constraint that the vectors assigned to a literal and its negation sum to $0$).     
\begin{align*}
 \tilde{q}(c) = \max_{\norm{v_j}_2^2=1 } \fr{m}\sum_{i \in [m]}\sum_{j \neq {j'} \in\ c_i} v_jv_{j'} 
\end{align*}
We solve the semi-definite program efficiently (up to small additive error which we will ignore) and accept instances where 
\[ \tilde{q}(c) > K_G(a_1t\log(t) - t).\]

Grothendieck's inequality implies that the integrality gap of this relaxation is  a constant; there exist $K_G \in [1.5, 2]$ such that
\begin{align}
    \label{eq:gi}
    q(c) - t \leq \tilde{q}(c)  \leq K_G (q(c) -t).
\end{align} 

For such instances, Equation \eqref{eq:gi} implies that $q(c) > a_1t\log(t)$ since
\[ q(c) - t \geq \frac{\tilde{q}(c)}{K_G} > a_1t\log(t) -t. \]
 For instances that we reject, it holds that
 \[ q(c) - t \leq \tilde{q}(c) < K_G(a_1t\log(t) - t) \]
Since $K_G \in [1.5,2]$,  we have $q(c) \leq 2a_1t\log(t)$.
\end{proof}

We refer to instances rejected by the algorithm as {\em pseudorandom}. By Markov's inequality applied to the definition of $q(c)$, we have the following claim:
\begin{lemma}
    Given pseudorandom $c \in C^m$ and $z \in Z$, for every $\delta > 0$, there are at most $\delta m$ clauses such that $|c_i \cdot z| \geq \sqrt{a_1t\log(t)/\delta}$.
\end{lemma}

We also have the following lemma which we state without proof
\begin{lemma}
    Every functions $g: \{-d, \ldots,d\} \to \pmo$ can be written as a polynomial in $x$ of degree $2d$ with coefficients bounded by $\exp(d\log(d))$. 
\end{lemma}

We now complete the proof of Theorem \ref{thm:sigmoid-hard}.

\begin{proof}
Let $\eta\in (0,1/2)$ be the constant guaranteed to exist by \Cref{assumption:2}. We define $\alpha = 1/4 - \eta/2\in (0,1/4)$. We fix an arbitrary constant $\varepsilon\in (0,1/3)$. Throughout the proof, we will treat $\eta, \alpha$, and $\varepsilon$ as fixed constants (that can hide in big-$O$ notations). Consider an algorithm $A$ for $(\alpha, k^{-C})$-weak agnostic learning for $\Sigma_{L}$ over $\Dk$ with running time $k^{(\log(L))^{1- \eps}}$ and success probability at least $3/4$ for any sufficiently large $k$ and $L:= \exp(\log^{c} k)$ for some $c\in (0,1)$. It suffices to use $A$ to efficiently refute random $t$-XorSat with parameter $\eta$ for any sufficiently large $n$ and $t = \lfloor \log ^s(n)\rfloor$ in time $n^{o(t)}$, where $s=2c/(1 - c + \varepsilon) > 0$.

    We view the clauses in a $t$-XorSat problem as a distribution over $C \subseteq \zo^{2n}$, with the $c_i$s being points and $b_i$ their labels. Let $d = \Theta(\sqrt{(t\log t)/\alpha}) = \Theta(\sqrt{t\log t})$ be such that at most $\alpha m$ clauses fail to satisfy $|c_i\cdot z| \leq d/2$. We consider the low degree feature expansion of $C$ denoted $C^{\otimes d}$ which contains a monomial $\prod_{j \in T}c_i^{(j)}$ for every $T \subseteq [2n]$ of size at most $d$, so that $C^{\otimes d} \subseteq \zo^k$ for 
    \begin{equation} 
    \label{eq:sigmoid-def-k}
    k = {2n \choose \leq d} := \sum_{j=0}^{d}{2n \choose j} .
    \end{equation}
    Since every $c \in C$ has weight exactly $t$, every $c^{\otimes d} \in C^{\otimes d}$ has weight $B_1 = {t \choose \leq d} = \exp(O(d\log(d)))$. This lets us write $c_i^{\otimes d} = B_1 \bv_i$ where $\bv_i \in \Delta_k$. Thus an instance $I$ which gives a distribution on $(c_i, b_i)$ where $c_i \in C$ $b_i \in \pmo$ also gives a distribution over $(\bv_i, b_i) \in \Dk \times \pmo$.

    There exists a degree $d$ polynomial 
    \[    p(t) = \sum_{j=0}^{d}\alpha_jt^j\] such that $|\alpha_j| =\exp(O(d\log(d)))$ and $p(c_i\cdot z) = \Xor_{c_i}(z) \ \text{for} \ |c_i\cdot z| \leq d/2$. If $|c_i\cdot z| \geq d/2$ then $p(c_i\cdot z) \in \R$. Since $c_i \in \{0,1\}^n$, we can multilinearize the terms of the form $(c_i\cdot z)^j$ as 
    \[(c_i\cdot z)^j = \sum_{T \subseteq [n], |T| \leq j} w^j_T\prod_{i \in T}c_i \]
    for coefficients $w^j_T = \exp(O(j\log j))$.
    So we can write
    \begin{align*}
        p(c_i \cdot z) &= \sum_{j=0}^{2d}\alpha_j(c_i\cdot z)^j\\
        &= \sum_{j=0}^{2d}\alpha_j\sum_{T \subseteq [n], |T| \leq j} w^j_T\prod_{i \in T}c_i\\
         &= \sum_{T\subseteq [n], |T| \leq 2d}\lt(\sum_{j \geq |T|}\alpha_jw^j_T\rt)\prod_{i \in T}c_i\\
         &= \sum_{T\subseteq [n], |T| \leq 2d}w_T'\prod_{i \in T}c_i\\
         & = w'\cdot c^{\otimes d}
    \end{align*}
    for coefficients $w'_T$ bounded in absolute value by $|w'_T| = \exp(O(d\log d))$.

    We renormalize $w'$ to be bounded in $[-1,1]^k$. We write  
    $w' =  wB_1$ for $B_1 = \max_T |w'_T| = \exp(O(d\log d)$.
     We have
    \[P(c_i\cdot z) =  w'\cdot c_i^{\otimes d} = B_1 {t \choose d} w\cdot \bv_i. \]
    Hence if $|c_i\cdot z| \leq d/2$, then the quantity above equals $\Xor_{c_i}(z) \in \pmo$, else it takes on values in $\R$. 
    
    By our choice of $d = \Theta(\sqrt{t\log t})$, we have
    \[
    \log \left(B_1{t \choose d}\right) = \log B_1 + O(\log (d \log t)) = O(d \log d) \le t^{1/2 + o(1)}.
    \]
    By our choice of $L:= \exp(\log^ck)$, we have
    \begin{align}
    \log L & = (\log k)^c \notag \\
    &  = (d \log n)^{c + o(1)}\tag{by \eqref{eq:sigmoid-def-k}}\\
    & = d^{c + o(1)}(\log n)^{c + o(1)}\notag\\
    &  = t^{c/2 + o(1)} (\log n)^{s(1 - c + \varepsilon)/2 + o(1)} \tag{by $d = \Theta(\sqrt{t\log t})$ and $s = 2c/(1 - c + \varepsilon)$}\\
    & = t^{c/2 + o(1)} t^{(1 - c + \varepsilon)/2 + o(1)} \tag{by $t = \lfloor \log ^s n\rfloor$}\\
    & = t^{1/2 + \varepsilon/2 + o(1)}.\label{eq:sigmoid-logL}
    \end{align}
    
    Therefore, for sufficiently large $n$,
    \[
    L \ge a B_1{t \choose d},
    \]
    for some constant $a$ so that $\tanh(a) \geq 1 - \alpha$, where we use our choice of constant $\alpha := 1/4 - \eta/2 > 0$.

    Consider the function $g(\bv) = \tanh(L' w\cdot \bv) \in \Sigma_{L}$. We can find $a' \ge a$ such that $L = a'B_1{t\choose d}$. We have for each $i \in [m]$,
    \begin{align*}
        g(\bv_i) &= \tanh(Lw \cdot \bv_i)\\
        &= \tanh\lt(a' B_1 w \cdot {t \choose \leq d}\bv_i\rt)\\
        &= \tanh(a' w'\cdot c_i^{\otimes d})\\
        &= \tanh(a' p(c_i\cdot z)).
    \end{align*}
Therefore, $g(\bv_i) \ge 1- \alpha$ if $p(c_i\cdot z) = 1$, and  $g(\bv_i) \le -1+ \alpha$ if $p(c_i\cdot z) = -1$. Moreover, it is clear that $g(\bv_i)\in [-1,1]$ always holds.
    
    Recall our definition $\alpha := 1/4 - \eta/2 > 0$.
    If $\val(I) \geq 1 - \eta$, then by taking the function $g$ derived from $z$ such that $\val(z, I)\geq 1 - \eta = 1/2 + 2\alpha$, excluding the at most $\alpha m$ clauses that fail to satisfy $|c_i\cdot z| \leq d/2$, we get $g \in \Sigma_L$ such that
    \[ \fr{m}\sum_{i=1}^mg(\bv_i)b_i \geq (1/2 + \alpha) \times (1 - \alpha) + (1/2 - \alpha) \times (-1) \geq \alpha.\]
    
    Thus based on the methodology of \cite{average-improper} (see e.g.\ Theorem 2.1 in \cite{amit-daniely}), we can apply our weak agnostic learning algorithm $A$ to efficiently distinguish the case with $\val(I) \ge 1 - \eta$ and the case with uniformly random clauses with success probability at least $3/4$, solving the $t$-XorSat refutation problem.

    As long as the running time of algorithm $A$ is bounded by $k^{\log(L)^{1- \eps}}$ for some $\eps > 0$, the running time for $t$-$\XS$ refutation is bounded by
    \[ k^{O(\log(L)^{1- \eps})} \leq n^{O(d\log(L)^{1- \eps})}.\]
By \eqref{eq:sigmoid-logL}, we have
\[
d\log(L)^{1- \eps} = t^{1/2 + o(1)} t^{(1 -\eps)(1/2 +\eps/2 + o(1))} = t^{1 - \varepsilon^2/2 + o(1)}.
\]
Thus the running time for $t$-$\XS$ refutation is bounded by $n^{o(t)}$, as desired.
\end{proof}

%% file: appendix.tex
\label{app:red}

\section{Proofs from Section~\ref{sec:kernel}}
\label{sec:proof-kernel}
\subsection{Proof of Theorem~\ref{thm:kernel-weak}}
We break the proof in a sequence of lemmas, starting with simplifying the objective function.

\begin{lemma}
\label{lem:corr1}
    Let $w_0 = \E_\mD[z\varphi_\bv]$. 
 Then $w_0 \in B_\Gamma(s)$ and for any $w \in \Gamma$ we have
    \begin{align}
    \label{eq:obj}
        \E_{(\bv, z)\sim \mD}[ w(\bv)z] = \ip{w}{w_0}_\Gamma.
    \end{align}
\end{lemma}
\begin{proof}
For any $w \in \Gamma$ we can write the correlation objective as
    \begin{align*}
        \E_{(\bv, z)\sim \mD}[ w(\bv)z] = \E[\ip{w}{\varphi_\bv}_\Gamma z] = \ip{w}{\E[z\varphi_\bv]}_\Gamma 
        = \ip{w}{w_0}_\Gamma.
    \end{align*}   
To bound its norm, observe that
    \begin{align*}
        \|w_0\|_\Gamma = \|\E_\mD[z\varphi_\bv]\|_\Gamma \leq \max_{\bv \in \Delta_k, z \in \pmo} \|z\varphi_\bv\|_\Gamma \leq s.
    \end{align*}
\end{proof}

    Next we show that we can approximate $w_0$ uniformly from samples by the function
    \[ \tilde{w}_0 = \frac{1}{n}\sum_{i=1}^nz_i\varphi_{\bv_i}. \]

\begin{lemma}
\label{lm:hilbert-concentration}
    For any $\delta\in (0,1/2)$, for some $n_0 = O(r^2s^2\alpha^{-2}\log(1/\delta))$, for any $n \ge n_0$ and any $(\bv_1,z_1),\ldots,\allowbreak (\bv_n,z_n)$ drawn i.i.d.\ from $\mD$, with probability at least $1-\delta$,
    \begin{align} 
    \label{eq:good-event}
        \norm{\tilde{w}_0 - w_0}_\Gamma \leq \frac{\alpha}{3r}.
    \end{align}
\end{lemma}
The proof uses McDiarmid's inequality (see e.g.\ Lemma 26.4 of \cite{SSbook}).
\begin{proof}
        We can write
        \begin{align*} 
    \|\tilde{w}_0 - w_0\|_\Gamma = \norm{\sum_{i=1}^n \frac{z_i\varphi_{\bv_i}}{n} - w_0}_\Gamma = \fr{n}\norm{\sum_{i=1}^n (z_i\varphi_{\bv_i} - w_0)}_\Gamma
        \end{align*} 
        Since each term $z_i\varphi_{\bv_i} - w_0$ has expectation $0$, and the terms are independent, for $i \neq j$ 
        \[ \E[\ip{z_i\varphi_{\bv_i} - w_0}{z_j\varphi_{\bv_j} - w_0}_\Gamma] = 0\]
        Hence we can bound
\begin{align*}
 \E[\|\tilde{w}_0 - w_0\|_\Gamma^2]
&= \fr{n^2}\E\left[\sum_{i=1}^n {\|z_i\varphi_{\bv_i} - w_0\|_\Gamma^2} + \sum_{i \neq j} \ip{z_i\varphi_{\bv_i} - w_0}{\z_j\varphi_{\bv_j} - w_0}\right]\\
& = \frac 1n \E[\|z_1\varphi_{\bv_1} - w_0\|_\Gamma^2]\\
& \le 4s^2/n \le (\alpha/(6r))^2.
\end{align*}
by our choice of $n$. By convexity,
\begin{align}
\label{eq:exp1}
\E[\|\tilde{w}_0 - w_0\|_\Gamma] \leq \frac{\alpha}{6r}.
\end{align}
        Note that each i.i.d.\ term $z_i\varphi_{\bv_i}$ in the definition of $\tilde{w}_0$ has norm $\|z_i\varphi_{\bv_i}\|_\Gamma \le s$, so by McDiarmid's inequality, with probability at least $1 -\delta$, 
\begin{equation}
\label{eq:kernel-2}
\big |\norm{\tilde{w}_0 - w_0}_\Gamma - \E[\norm{\tilde{w}_0 - w_0}_\Gamma] \big| \le  \alpha /(6r). 
\end{equation}
Combining this with Equation \eqref{eq:exp1} gives the desired claim.
\end{proof}
We need the following simple helper lemma to finish proving \Cref{thm:kernel-weak}:
\begin{lemma}
\label{lm:inner-lower}
Let $w,\tilde w$ be elements of a Hilbert space $\Gamma$. If $\tilde w \ne 0$, define $\bar w = \tilde w/\|\tilde w\|_\Gamma$. If $\tilde w = 0$, define $\bar w$ to be an arbitrary element of $B_\Gamma(1)$. Then
\[
\langle w, \bar w\rangle \ge \|\tilde w\|_\Gamma - \|w - \tilde w\|_\Gamma\ge \|w\|_\Gamma - 2\|w - \tilde w\|_\Gamma.
\]
\end{lemma}
\begin{proof}
We have
\begin{align*}
\langle w,\bar w\rangle \ge \langle \tilde w, \bar w\rangle - \|w - \tilde w\|_\Gamma & = \|\tilde w\|_\Gamma - \|w - \tilde w\|_\Gamma \ge \|w\|_\Gamma - 2\|w - \tilde w\|_\Gamma.\qedhere
\end{align*}
\end{proof}
\begin{proof}[Proof of \Cref{thm:kernel-weak}]
In the weak agnostic learning task, we assume that there exists $w \in B_\Gamma(r)$ so that
    \[ \E_{(\bv, \z)\sim \mD}[w(\bv) z] \ge \alpha.\]
Under this assumption, \Cref{lem:corr1} tells us that $\langle w_0,w\rangle_\Gamma \ge \alpha$. Since $w\in B_\Gamma(r)$, we have $\|w\|_\Gamma\le r$, and by the Cauchy-Schwarz inequality, $r\, \|w_0\|_\Gamma\ge \|w_0\|_\Gamma \|w\|_\Gamma \ge \langle w_0,w\rangle_\Gamma$.
Therefore, we can assume that $\|w_0\|_\Gamma \ge \alpha/r$ in the weak agnostic learning task.

\Cref{lm:hilbert-concentration} ensures that \eqref{eq:good-event} holds with probability at least $1-\delta$.
As long as \eqref{eq:good-event} holds,
by \Cref{lm:inner-lower} we have
\begin{equation}
\label{eq:kernel-weak-1}
\left\langle w_0, \frac{\tilde w_0}{\|\tilde w_0\|_\Gamma}\right\rangle \ge \|w_0\|_\Gamma - 2\|w_0 - \tilde w\|_\Gamma \ge \frac \alpha {3r}.
\end{equation}
The output $w_2$ of \Cref{alg:kernel-weak} can be expressed as
\begin{equation}
\label{eq:kernel-weak-2}
w_2 = \frac{\tilde w_0}{s\  \|\tilde w_0\|_\Gamma}.
\end{equation}
Combining \eqref{eq:kernel-weak-1} and \eqref{eq:kernel-weak-2}, we know that with probability at least $1-\delta$,
\[
\langle w_0,w_2\rangle \ge \frac \alpha{3rs}.
\]
By \Cref{lem:corr1}, the inequality above implies the weak learning guarantee, namely, $\E[w_2(\bv)z] \ge \alpha/(3rs)$. 
Finally, it is clear that $\|w_2\|_\Gamma \le 1/s$, so $w_2\in B_\Gamma(1/s)$, as desired.
\end{proof}
\subsection{Proof of \Cref{thm:kernel-audit}}
Consider a distribution $\mD$ of $(\bv,\y)\in \Delta_k \times \mE_k$, and define $\z:= \y - \bv$. Define $w_0\sps j:= \E [\z\sps j\varphi_\bv]$, where $\z\sps j$ is the $j$-th coordinate of $\z$. For $n$ i.i.d.\ data points $(\bv_1,\y_1),\ldots,(\bv_n,\y_n)$, define $\z_i := \y_i - \bv_i$. Define $\tilde w_0\sps j:= \frac 1n \sum_{i=1}^n \z_i\sps j\varphi_{\bv_i}$.
\begin{lemma}
\label{lm:concentration-audit}
When $n\ge Ckr^2s^2\varepsilon^{-2}\log(1/\delta)$, with probability at least $1-\delta$,
\begin{equation}
\label{eq:concentration-audit}
\sum_{j=1}^k\|\tilde w_0\sps j - w_0\sps j\|_\Gamma \le \alpha / (3r).
\end{equation}
\end{lemma}
\begin{proof}
We first show that
\[
\sum_{j=1}^k\E\|\tilde w_0\sps j - w_0\sps j\|_\Gamma \le \alpha/(6r).
\]
For every $j$,
\begin{align*}
\E[\|\tilde w_0\sps j - w_0\sps j\|_\Gamma^2] & = \E\left[\left\| \frac 1n \sum_{i=1}^n \z_i\sps j\varphi_{\bv_i} - w_0\sps j\right\|_\Gamma^2\right] \\
& = \frac 1{n^2}\sum_{i=1}^n \E\|\z_i\sps j\varphi_{\bv_i} - w_0\sps j\|_\Gamma^2 + \frac 1{n^2}\sum_{i\ne i'} \langle \z_i\sps j\varphi_{\bv_i} - w_0\sps j, \z_{i'}\sps j\varphi_{\bv_{i'}} - w_0\sps j \rangle \\
& = \frac 1{n^2}\sum_{i=1}^n \E\|\z_i\sps j\varphi_{\bv_i} - w_0\sps j\|_\Gamma^2\\
& = \frac 1n \E[\|\z\sps j_1\varphi_{\bv_1} - w_0\sps j\|_\Gamma^2]\\
& \le \frac 1n \E[\|\z\sps j_1\varphi_{\bv_1}\|_\Gamma^2]\\
& \le \frac {s^2}n\E[(\z_1\sps j)^2]
\end{align*}
By Cauchy-Schwarz,
\[
\sum_{j=1}^k\sqrt{\E[(\z_1\sps j)^2]} \le \sqrt{k\sum_{j=1}^k\E[(\z_1\sps j)^2]} \le \sqrt{k\sum_{j=1}^k\E|\z_1\sps j|} \le \sqrt{2k}.
\]
Therefore,
\[
\sum_{j=1}^k\E\|\tilde w_0\sps j - w_0\sps j\|_\Gamma\le \sum_{j=1}^k\sqrt{\E[\|\tilde w_0\sps j - w_0\sps j\|_\Gamma^2]} \le \frac {s\sqrt {2k}}{\sqrt n}  \le \alpha/(6r).
\]
Finally, we apply McDiarmid's inequality to the following function of $(\z_1,\bv_1),\ldots,(\z_n,\bv_n)$:
\[
\sum_{j=1}^k\|\tilde w_0\sps j - w_0\sps j\|_\Gamma = \sum_{j=1}^k\left\|\frac 1n \sum_{i=1}^n\z_i\sps j\varphi_{\bv_i} - w_0\sps j\right\|_\Gamma
\]
and get that with probability at least $1-\delta$,
\[
\sum_{j=1}^k\|\tilde w_0\sps j - w_0\sps j\|_\Gamma \le \sum_{j=1}^k\E\|\tilde w_0\sps j - w_0\sps j\|_\Gamma + \alpha/(6r) \le \alpha/(3r). \qedhere
\]
\end{proof}
\begin{proof}[Proof of \Cref{thm:kernel-audit}]
In the auditing task, we assume that there exists $w\in B_\Gamma(r)^k$ such that
\[
\E[\langle \y - \bv, w(\bv)\rangle] \ge \alpha.
\]
Using our definition of $\z:=\y - \bv$ and $w_0\sps \ell:= \E[\z\sps \ell \varphi_\bv]$, by \Cref{lem:corr1} we have
\[
\sum_{\ell =1}^k\|w_0\sps \ell \|_\Gamma \ge \frac 1r \sum_{\ell = 1}^k\langle w_0\sps \ell, w\sps \ell\rangle_\Gamma = \frac 1r \sum_{\ell = 1}^k \E[\z\sps j w\sps j(\bv)] = \frac 1r \E[\langle \y - \bv, w(\bv)\rangle] \ge \alpha/r.
\]
\Cref{lm:concentration-audit} ensures that \eqref{eq:concentration-audit} holds with probability at least $1-\delta$. 
Define $\bar w_0:= \tilde w_0/\|\tilde w_0\|_\Gamma$ if $\tilde w_0\ne 0$, and define $\bar w_0 := 0$ if $\tilde w_0 = 0$.
As long as \eqref{eq:concentration-audit} holds, by \Cref{lm:inner-lower},
\[
\sum_{\ell =1}^k\langle w_0\sps \ell,\bar w_0\sps \ell \rangle \ge \sum_{\ell =1}^k \|w_0\sps \ell\|_\Gamma - 2\sum_{\ell =1}^k \|\tilde w_0\sps \ell - w_0\sps \ell\|_\Gamma \ge \alpha/(3r).
\]
In \Cref{alg:kernel-audit}, we have $w_2\sps \ell = \bar w_0\sps \ell/s$, and thus the inequality above implies
\[
\sum_{\ell =1}^k\langle w_0\sps \ell,w_2\sps \ell \rangle \ge \alpha/(3rs).
\]
Therefore by \Cref{lem:corr1}, with probability at least $1-\delta$,
\[
\E[\langle \y - \bv, w_2(\bv)\rangle] = \sum_{\ell=1}^k \E[\z\sps \ell w_2\sps \ell(\bv)] = \sum_{\ell=1}^k \langle w_0\sps \ell,w_2\sps \ell \rangle \ge \alpha/(3rs).
\]
This proves that the output $w_2$ of \Cref{alg:kernel-audit} satisfies the requirement of the auditing task. Finally, it is clear that each $w_2\sps \ell$ has norm $\|w_2\sps \ell\|_\Gamma \le 1/s$, so $w_2\in B_\Gamma(1/s)^k$, as desired.
\end{proof}